\documentclass[10pt]{article}
\usepackage{times}

\pdfoutput=1

\usepackage{microtype}
\usepackage{graphicx}
\usepackage{subcaption}
\usepackage{booktabs}
\usepackage{multirow}

\usepackage{amsfonts,amsmath,amssymb,amsthm}

\usepackage{color}
\usepackage{algorithm}
\usepackage{algorithmic}

\usepackage{tabu}
\usepackage{url}

\usepackage{authblk}

\usepackage{parskip}

\newlength{\defbaselineskip}
\usepackage{etoolbox}
\newtoggle{arxiv}
\toggletrue{arxiv}

\usepackage{mathtools}
\usepackage[dvipsnames]{xcolor}

\usepackage{tkz-euclide}

\usepackage[numbers,sort]{natbib}

\usepackage{hyperref}
\usepackage{cleveref}

\newlength{\myfcwidth}
\newlength{\mydiagwidth}
\newtheorem{theorem}{Theorem}[section]

\newtheorem{lemma}[theorem]{Lemma}
\newtheorem{proposition}[theorem]{Proposition}
\newtheorem{definition}[theorem]{Definition}
\newtheorem{example}{Example}[section]

\DeclareMathOperator\acosh{acosh}
\DeclareMathOperator\asin{asin}
\DeclareMathOperator\tr{trace}

\newcommand{\R}{\mathbb{R}}


\begin{document}

\title{Representation Tradeoffs for Hyperbolic Embeddings}
% \author{Christopher De Sa \and Megan Leszczynski \and Jian Zhang \and Alana Marzoev \and Christopher R. Aberger \and Kunle Olukotun \and Christopher R{\'e}}

\author[$\ddagger$]{Christopher De Sa}
\author[$\dagger$]{Albert Gu}
\author[$\dagger$]{Christopher R{\'e}}
\author[$\dagger$]{Frederic Sala}
\affil[$\dagger$]{Department of Computer Science, Stanford University}
\affil[$\ddagger$]{Department of Computer Science, Cornell University\vspace{4pt}}
\affil[ ]{\footnotesize{\texttt{cdesa@cs.cornell.edu}, \texttt{albertgu@stanford.edu}, \texttt{chrismre@cs.stanford.edu},\texttt{fredsala@cs.stanford.edu}}}

\maketitle

\begin{abstract}
Hyperbolic embeddings offer excellent quality with few dimensions when embedding hierarchical data structures like synonym or type hierarchies. Given a tree, we give a combinatorial construction that embeds the tree in hyperbolic space with arbitrarily low distortion without using optimization. On WordNet, our combinatorial embedding obtains a mean-average-precision of $0.989$ with only two dimensions, while Nickel et al.'s recent construction obtains $0.87$ using $200$ dimensions.  We provide upper and lower bounds that allow us to characterize the precision-dimensionality tradeoff inherent in any hyperbolic embedding. To embed general metric spaces, we propose a hyperbolic generalization of multidimensional scaling (h-MDS). We show how to perform exact recovery of hyperbolic points from distances, provide a perturbation analysis, and give a recovery result that allows us to reduce dimensionality. The h-MDS approach offers consistently low distortion even with few dimensions across several datasets. Finally, we extract lessons from the algorithms and theory above to design a PyTorch-based implementation that can handle incomplete information and is scalable.

\end{abstract}

\section{Introduction}
\label{sec:introduction}
Recently, hyperbolic embeddings have been proposed as a way to capture
hierarchy information for use in link prediction and natural language
processing tasks~\cite{fb, ucl}. These approaches are an exciting new
way to fuse rich structural information (for example, from knowledge graphs or
synonym hierarchies) with the continuous representations favored by
modern machine learning.

To understand the intuition behind hyperbolic embeddings' superior
capacity, note that trees can be embedded with arbitrarily low
distortion into the Poincar\'e disk, a model of hyperbolic space
with only two dimensions~\cite{sarkar}. In contrast, Bourgain's
theorem \cite{Lineal} shows that Euclidean space is unable to obtain
comparably low distortion for trees---even using an unbounded number of
dimensions.
Moreover, hyperbolic space can preserve certain properties;
for example, angles between embedded vectors are the same in
both Euclidean space and the Poincar\'e model (the mapping is conformal),
which suggests embedded data may be easily able to integrate with downstream
tasks. 

Many graphs, such as complex networks~\cite{krioukov2010hyperbolic}, including the Internet~\cite{krioukov2009curvature} and social networks~\cite{verbeek2016metric}) are known to have hyperbolic structure and thus befit hyperbolic embeddings. Recent works show that hyperbolic representations are indeed suitable for many hierarchies (e.g, the question answering system HyperQA proposed in \cite{tay2018hyperbolic}, vertex classifiers in \cite{ucl}, and link prediction \cite{fb}). However, the optimization problem underlying these embeddings is
challenging, and we seek to understand the subtle tradeoffs involved.

We begin by considering the situation in which we are given an input
graph that is a tree or nearly tree-like, and our goal is to produce
a low-dimensional hyperbolic embedding that preserves all distances. This
leads to a simple strategy that is combinatorial in that it does not
minimize a surrogate loss function using gradient descent. It is both
fast (nearly linear time) and has formal quality guarantees. The
approach proceeds in two phases: (1) we produce an embedding of a
graph into a weighted tree, and (2) we embed that tree into the
hyperbolic disk. In particular, we consider an extension of an elegant embedding of trees
into the Poincar\'e disk by Sarkar~\cite{sarkar} and recent work on
low-distortion graph embeddings into tree metrics~\cite{Abraham}. For trees, this approach has nearly perfect
quality. On the WordNet hypernym graph reconstruction, this obtains
nearly perfect mean average precision (MAP) $0.989$ using just two
dimensions, which outperforms the best published numbers in \citet{fb}
by almost $0.12$ points with $200$ dimensions.

We analyze this construction to extract fundamental tradeoffs. One tradeoff involves  the
dimension, the properties of the graph, and the number of bits of precision - an important hidden cost. For example, on the WordNet graph, we require almost 500 bits of precision to store values from the combinatorial embedding. We can reduce this number to 32 bits, but at the cost of using 10 dimensions instead of two. We show that for a fixed precision, the dimension
required scales linearly with the length of the longest path. On the
other hand, the dimension scales logarithmically with the maximum
degree of the tree. This suggests that hyperbolic embeddings should
have high quality on hierarchies like WordNet but require large
dimensions or high precision on graphs with long chains---which is supported by our
experiments. A second observation is that in contrast to Euclidean
embeddings, hyperbolic embeddings are not scale invariant. This
motivates us to add a learnable scale term into a stochastic gradient descent-based Pytorch algorithm
described below, and we show that it allows us to empirically improve the
quality of embeddings.

To understand how hyperbolic embeddings perform for metrics that are
far from tree-like, we consider a more general problem: given a matrix
of distances that arise from points that are embeddable in hyperbolic
space of dimension $d$ (not necessarily from a graph), find a set of
points that produces these distances. In Euclidean space, the problem
is known as multidimensional scaling (MDS) which is solvable using
PCA.%
\footnote{There is no perfect analogue of PCA in hyperbolic
space~\cite{annals:stats}.}
A key step is a transformation that
effectively centers the points--without knowledge of their exact
coordinates. It is not obvious how to center points in hyperbolic
space, which is curved.
% We show that in hyperbolic space, a centering
% operation is possible using the Perron-Frobenius theorem. In
% particular, the largest eigenvalue of the distance matrix is positive
% and corresponds to a component-wise positive eigenvector. The
% components of this eigenvector allow us to define a transformation to
% center the points.
We show that in hyperbolic space, a centering operation is still possible with respect to a non-standard mean.
In turn, this allows us to reduce the hyperbolic
MDS problem (h-MDS) to a standard eigenvalue problem, and so
it can be solved with scalable power methods.
Further, we extend
classical perturbation analysis~\cite{Sibson1,Sibson2}. When applied to distances from real data,
h-MDS obtains low distortion on graphs that are far from tree
like. However, we observe that these solutions may require high
precision, which is not surprising in light of our previous analysis.

Finally, we consider handling increasing amounts of noise in the
model, which leads naturally into new SGD-based formulations. In
traditional PCA, one may discard eigenvectors that have
correspondingly small eigenvalues to cope with noise. In hyperbolic
space, this approach may produce suboptimal results. Like PCA, the
underlying problem is nonconvex. In contrast to PCA, the optimization
problem is more challenging: the underlying problem has local minima
that are not global minima. Our main technical result is that an
SGD-based algorithm initialized with a h-MDS solution can recover the
submanifold the data is on--even in some cases in which the data is
perturbed by noise that can be full dimensional. Our algorithm
essentially provides new recovery results for convergence for
Principal Geodesic Analysis (PGA) in hyperbolic space~\cite{PGA, GPCA}.
We discuss the nuances between our optimization algorithms and previous attempts 
at these problems in Appendix~\ref{sec:related}.

All of our results can handle incomplete distance information through standard
techniques. Using the observations above, we
implemented an SGD algorithm that minimizes
the loss derived from the PGA loss using PyTorch.\footnote{A minor
  instability with \citet{fb,ucl}'s formulation is that one must guard
  against \textsc{NaN}s. This instability may be unavoidable in
  formulations that minimize hyperbolic distance with gradient
  descent, as the derivative of the hyperbolic distance has a
  singularity, that is, $\lim_{y \to x} \partial_x |d_H(x, y)| \to
  \infty$ for any $x \in \mathbb{H}$ in which $d_H$ is the hyperbolic
  distance function. This issue can be mitigated by minimizing
  $d_H^2$, which does have a continuous derivative throughout
  $\mathbb{H}$. We propose to do so in Section~\ref{sec:PGA} and
discuss this further in the Appendix.}

%% In Section~\ref{sec:background}, we give background on hyperbolic
%% geometry and embeddings. We describe our results on combinatorial
%% embeddings in Section~\ref{sec:combinatorial}. In
%% Section~\ref{sec:MDS}, we introduce hMDS and PGA. Experimental results
%% in Section~\ref{sec:experiments} illustrate the regimes where they are
%% most useful.

\section{Background}
\label{sec:background}
We provide intuition connecting hyperbolic space and tree distances, discuss the metrics used to measure embedding fidelity, and provide the relationship between reconstruction and learning for graph embeddings.

\begin{figure}
\centering
\begin{minipage}[b]{0.35\textwidth}
\begin{tikzpicture}[scale=2.0]
\tkzDefPoint(0,0){O}
\tkzDefPoint(1,0){A}
\tkzDefPoint(-1,0){B}
\tkzDefPoint(0.7071,0.7071){C}
\tkzDefPoint(-0.7071,-0.7071){Cn}
\tkzDefPoint(-.2,0.9798){R}
\tkzDefPoint(.2,-0.9798){Rn}
\tkzDrawCircle(O,A)

\tkzDefPoint(0.7,0.3){a}
\tkzDefPoint(0.8,0.1){b}
\tkzDefPoint(-.5,.5){w1}
\tkzDefPoint(-0.1,.9){v1}
 
\tkzDefPoint(-.6,.4){w2}
\tkzDefPoint(0,.95){v2}

\tkzDefPoint(-.7,.3){w3}
\tkzDefPoint(.1,.9){v3}
 
\tkzDefPoint(-.8,.2){w4}
\tkzDefPoint(.2,.85){v4}

\tkzDefPoint(0.2,0){t1}
\tkzDefPoint(0.4,0){t2}
\tkzDefPoint(0.6,0){t3}
\tkzDefPoint(0.8,0){x}
 
\tkzDefPoint(0.1414, 0.1414){u1}
\tkzDefPoint(0.2828, 0.2828){u2}
\tkzDefPoint(0.4243, 0.4243){u3}
\tkzDefPoint(0.5657, 0.5657){y}

\tkzClipCircle(O,A)
\tkzDrawCircle[orthogonal through=u1 and t1](O,A)
\tkzDrawCircle[orthogonal through=u2 and t2](O,A)
\tkzDrawCircle[orthogonal through=u3 and t3](O,A)
\tkzDrawCircle[orthogonal through=x and y](O,A) 
\tkzDrawPoints[color=black,fill=black,size=12](t1,t2,t3,y,u1,u2,u3,x)

\tkzDrawPoints[color=black,fill=black,size=12](O)
\tkzDrawLine[through=O](R,Rn) 
\tkzLabelPoints[right](y)
\tkzLabelPoints[above](x)
\tkzLabelPoints[above left](O)
    
\end{tikzpicture}
  \end{minipage}
  \begin{minipage}[b]{0.50\textwidth}
\includegraphics[width=0.9\textwidth]{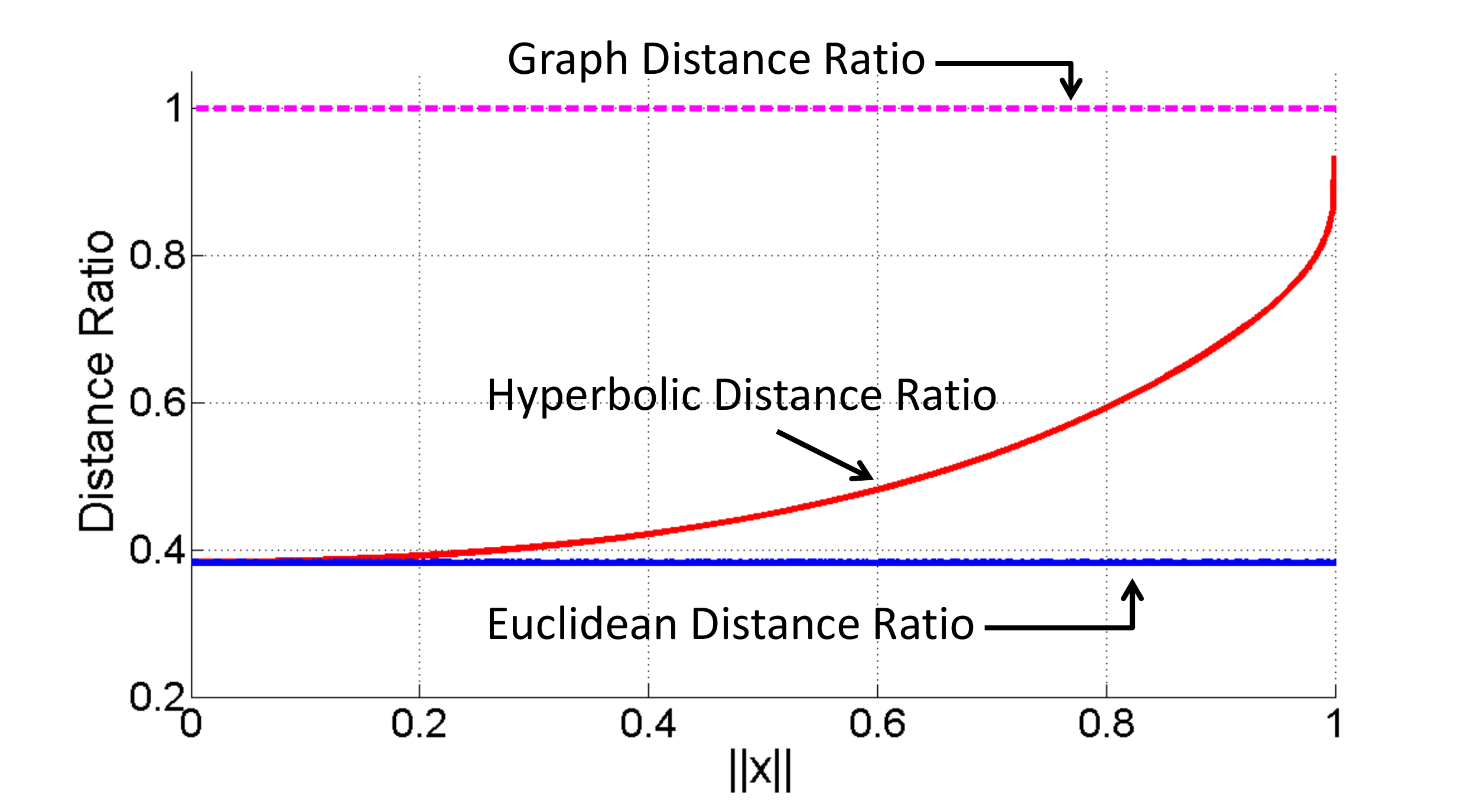}
  \end{minipage}
\caption{Geodesics and distances in the Poincar\'{e} disk. As $x$ and $y$ move towards the outside of the disk (i.e., letting $\|x\|, \|y\| \rightarrow 1$), the distance $d_H(x,y)$ approaches $d_H(x,O)+d_H(O,y)$.} %Right: circle inversion takes $X$ to $X'$.}
\label{fig:geod}
\end{figure}

\paragraph*{Hyperbolic spaces} The Poincar\'{e} disk $\mathbb{H}_2$ is a
two-dimensional model of hyperbolic geometry with points located in
the interior of the unit disk, as shown in Figure~\ref{fig:geod}. A
natural generalization of $\mathbb{H}_2$ is the Poincar\'{e} ball
$\mathbb{H}_r$, with elements inside the unit ball. The Poincar\'{e}
models offer several useful properties, chief among which is mapping
conformally to Euclidean space. That is, angles are preserved between
hyperbolic and Euclidean space. Distances, on the other hand, are not
preserved, but are given by
\[d_{H}(x,y) = \text{acosh} \left( 1 + 2 \frac{\lVert { x}-{ y} \rVert^2}{(1-\lVert { x}\rVert^2)(1-\lVert { y}\rVert^2)} \right).\]

There are some potentially unexpected consequences of this formula,
and a simple example gives intuition about a key technical property
that allows hyperbolic space to embed trees. Consider three
points: the origin $0$, and points $x$ and $y$ with $\|x\|=\|y\| = t$ for some
$t > 0$. As shown on the right of Figure~\ref{fig:geod}, as
$t\rightarrow1$ (i.e., the points move towards the outside of the
disk), in flat Euclidean space, the ratio
$\frac{d_E(x,y)}{d_E(x,0)+d_E(0,y)}$ is {\em constant} with respect to
$t$ (blue curve). In contrast, the ratio $\frac{d_H(x,y)}{d_H(x,0)+d_H(0,y)}$ approaches 1, or, equivalently, the distance $d_H(x,y)$ approaches
$d_H(x,0)+d_H(0,y)$ (red and pink curves). That is, the shortest path between $x$ and $y$ is
almost the same as the path through the origin. This is analogous to the property of trees
in which the shortest path between two sibling nodes is the path through their parent.
This tree-like nature of hyperbolic space is the key property exploited by embeddings. Moreover,
this property holds for arbitrarily small angles between $x$ and
$y$.

\paragraph*{Lines and geodesics}  There are two types of geodesics (shortest
paths) in the Poincar{\'e} disk model of hyperbolic space: segments of circles that are orthogonal to
the disk surface, and disk diameters \cite{GeometryText}. Our algorithms and
proofs make use of a simple geometric fact: {\em isometric} reflection across geodesics
(preserving hyperbolic distances) is represented in this Euclidean model as a \emph{circle inversion}.
A particularly important reflection associated with each point $x$ is the one that takes $x$ to the
origin \citep[p.~268]{GeometryText}.

\paragraph*{Embeddings and fidelity measures} An \emph{embedding} is a mapping $f: U \rightarrow V$ for spaces $U,V$ with distances $d_U, d_V$. We measure the quality of embeddings with several \emph{fidelity measures}, presented here from most local to most global.

Recent work \cite{fb} proposes using the \emph{mean average precision} (MAP). For a graph $G=(V,E)$, let $a \in V$ have neighborhood $\mathcal{N}_a = \{b_1, b_2, \ldots, b_{\operatorname{deg}(a)}\}$, where $\operatorname{deg}(a)$ denotes the degree of $a$. In the embedding $f$, consider the points closest to $f(a)$, and define $R_{a,b_i}$ to be the smallest set of such points that contains $b_i$ (that is, $R_{a,b_i}$ is the smallest set of nearest points required to retrieve the $i$th neighbor of $a$ in $f$). Then, the MAP is defined to be
\[
\text{MAP}(f) 
=
\frac{1}{|V|}\sum_{a \in V} \frac{1}{|\mathcal{N}_a|}\sum_{i=1}^{|\mathcal{N}_a|} \text{Precision}(R_{a,b_i})
=
\frac{1}{|V|}\sum_{a \in V} \frac{1}{\operatorname{deg}(a)}\sum_{i=1}^{|\mathcal{N}_a|} \frac{ \left| \mathcal{N}_a \cap R_{a,b_i} \right| }{\left| R_{a,b_i} \right|}.
\]
We have $\text{MAP}(f) \leq 1$, with equality as the best case. 
Note that MAP is not concerned with the underlying distances at all, but only the ranks between the distances of immediate neighbors. It is a \emph{local} metric.

%% To capture more of the global structure of a graph, a natural generalization is $k$-MAP, where the neighborhood of node $a$ is expanded to all nodes with distance at most $k$ from $a$. %The differences between $k$-MAP for $k=1,2,\ldots$ reveal how much of the graph structure is being captured.

The standard metric for graph embeddings is distortion $D$. For an $n$ point embedding,
\[D(f) = \frac{1}{\binom{n}{2}} \left(\sum_{u,v \in U:u\neq v} \frac{| d_V(f(u),f(v)) - d_U(u,v)|}{d_U(u,v)}\right).\]
The best distortion is $D(f) = 0$, preserving the edge lengths exactly.
This is a \emph{global} metric, as it depends directly on the underlying distances rather than the local relationships between distances.
A variant of this, the worst-case distortion $D_{\mathrm{wc}}$, is the metric defined by
\[D_{\mathrm{wc}}(f) = \frac{\max_{u,v \in U: u \neq v}d_V(f(u),f(v))/d_U(u,v)}{\min_{u,v \in U : u\neq v} d_V(f(u),f(v))/d_U(u,v)} .\]
That is, the wost-case distortion is the ratio of the maximal expansion and the minimal contraction of distances. Note that scaling the unit distance does not affect $D_{\mathrm{wc}}$. The best worst-case distortion is $D_{\mathrm{wc}}(f) = 1$. 

The intended application of the embedding informs the choice of metric. For applications where the underlying distances are important, distortion is useful. On the other hand, if only rankings matter, MAP may suffice. This choice is important: as we shall see, different embedding algorithms implicitly target different metrics.

\paragraph*{Reconstruction and learning} In the case where we lack a full
set of distances, we can deal with the missing data in one of two
ways. First, we can use the triangle inequality to recover the missing
distances. Second, we can access the scaled Euclidean distances (the
inside of the $\acosh$ in $d_H(x,y)$), and then the resulting matrix can be
recovered with standard matrix completion techniques \cite{TaoMatrix}. Afterwards, we
can proceed to compute an embedding using any of the approaches
discussed in this paper. We quantify the error introduced by this
process experimentally in Section~\ref{sec:experiments}.
%process both theoretically and experimentally.

%\emph{matrix completion}. \yell{this is not true} If the original
%distance matrix was embeddable in $r$-dimensional hyperbolic space, it
%must have rank at most $r+2$. We can recover it using standard matrix
%completion techniques: for example, it is known that even in the
%presence of noise on the observed entries a rank $r$ $n \times n$
%matrix can be efficiently recovered to any fixed level of accuracy
%from $O(n r^3 \log n)$ samples using SGD under mild
%conditions~\cite{alecton}. Afterwards, we can proceed to compute an
%embedding using any of the techniques discussed in this paper
%Once we have recovered $Y$, we can use it to find all the entries of $d$, and then proceed to compute an embedding for $d$ using any of the techniques discussed in this paper.

%Since missing entries of $d$ correspond to missing entries of $Y$ (as $Y$ is just an entrywise function of $d$), we can recover the missing entries of $d$ by recovering the missing entries of $Y$.

\section{Combinatorial Constructions}
\label{sec:combinatorial}

%                I feel like we need some chart here to show how precision scales intuitively for folks (i.e, what’s in the experiments could come up here).

%                Hyperbolics produce low distortion and even low MAP for short bushy… also back this up with numbers from your experiments!
%                Say that the scale is critically important here. Explain that you can simply add in a learnable scale parameter. Mention a micro experiment where it has an impact, and then comment if you think it’s important more broadly.
%            Now you can devote as much as you want to the proof. 
%                You could just give the intuition of each (the precision and the lower bound).
%                I think it’s fine to say this is the key element of the proof.  
%        Embedding trees. Figure 4 is very good. 
%            You should push credit to Abraham earlier (when you mention Steiner nodes). Our contribution is applying them to these embeddings, we build on their iddeas.
%             Line 284. Probably make the is an example environment. Make sure it’s clear the reader needs to transition.
%            Make sure this section is a little more sandwich method (you don’t tell them up front—need more signposting)

We first focus on hyperbolic tree embeddings---a natural approach
considering the tree-like behavior of hyperbolic space.  We
review the embedding of \citet{sarkar} to higher dimensions. We then
provide novel analysis about the precision of the embeddings that
reveals fundamental limits of hyperbolic embeddings. In particular, we
characterize the bits of precision needed for hyperbolic
representations. We then extend the construction to $r$ dimensions,
and we propose to use Steiner nodes to better embed general graphs as
trees building on a condition from \citet{Abraham}.

\paragraph*{Embedding trees} The nature of hyperbolic space lends itself towards excellent tree embeddings. In fact, it is possible to embed trees into the Poincar\'{e} disk $\mathbb{H}_2$ with arbitrarily low distortion \cite{sarkar}. Remarkably, trees cannot be embedded into Euclidean space with arbitrarily low distortion for \emph{any} number of dimensions. These notions motivate the following two-step process for embedding hierarchies into hyperbolic space.
\begin{enumerate}
  \setlength\itemsep{0em}
\item Embed the graph $G=(V,E)$ into a tree $T$,
\item Embed $T$ into the Poincar\'{e} ball $\mathbb{H}_d$.
\end{enumerate}

We refer to this process as the \emph{combinatorial construction}. Note that we are not required to minimize a loss function. We begin by describing the second stage, where we extend an elegant construction from \citet{sarkar}. 

\subsection{Sarkar's Construction}
Algorithm~\ref{alg:sarkar} implements a simple embedding of trees into $\mathbb{H}_2$. The algorithm takes as input a scaling factor $\tau$ a node $a$ (of degree $\operatorname{deg}(a)$) from the tree with parent node $b$. Suppose $a$ and $b$ have already been embedded into $\mathbb{H}_2$ and have corresponding embedded vectors $f(a)$ and $f(b)$. The algorithm places the children $c_1, c_2, \ldots, c_{\operatorname{deg}(a)-1}$ into $\mathbb{H}_2$ through a two-step process. 

First, $f(a)$ and $f(b)$ are reflected across a geodesic (using circle inversion) so
that $f(a)$ is mapped onto the origin $0$ and $f(b)$ is mapped onto some point $z$.
% We compute the angle of $Z'$.
Next, we place the children nodes to vectors $y_1, \ldots, y_{d-1}$ equally spaced around a circle with radius $\frac{e^\tau-1}{e^\tau+1}$ (which is a circle of radius $\tau$ in the hyperbolic metric), and maximally separated from the reflected parent node embedding $z$. Lastly, we reflect all of the points back across the geodesic. 
Note that the isometric properties of reflections imply that all children are now at hyperbolic distance exactly $\tau$ from $f(a)$.

\begin{algorithm}[t]
\begin{algorithmic}[1]
\STATE \textbf{Input:} Node $a$ with parent $b$, children to place $c_1, c_2, \ldots, c_{\operatorname{deg}(a)-1}$, partial embedding $f$ containing an embedding for $a$ and $b$, scaling factor $\tau$
\STATE $(0, z) \leftarrow \operatorname{reflect}_{f(a) \rightarrow 0}(f(a),f(b))$ %\COMMENT{circle inversion}
\STATE $\theta \leftarrow \operatorname{arg}(z)$ \hspace{2em} \COMMENT{angle of $z$ from x-axis in the plane}
\FOR{$i \in \{1, \ldots, \operatorname{deg}(a)-1 \}$}
\STATE $y_i \leftarrow \left(\frac{e^\tau-1}{e^\tau+1} \cdot \cos\left(\theta + \frac{2\pi i}{\operatorname{deg}(a)} \right) , \frac{e^\tau-1}{e^\tau+1} \cdot \sin\left(\theta+\frac{2\pi i}{\operatorname{deg}(a)}\right) \right)$ % \label{alg:sarkar:step:circle}
\ENDFOR
\STATE $(f(a), f(b), f(c_1),\ldots,f(c_{\operatorname{deg}(a)-1})) \leftarrow \operatorname{reflect}_{0 \rightarrow f(a)}(0, z, y_1, \ldots, y_{\operatorname{deg}(x)-1})$
\STATE \textbf{Output:} Embedded $\mathbb{H}_2$ vectors $f(c_1), f(c_2), \ldots, f(c_{\operatorname{deg}(a)-1})$
\end{algorithmic}
\caption{Sarkar's Construction}
\label{alg:sarkar}
\end{algorithm}

To embed the entire tree, we place the root at the origin $O$ and its children in a circle around it (as in Step~5 of Algorithm~\ref{alg:sarkar}), then recursively place their children until all nodes have been placed. Notice this construction runs in linear time.

\subsection{Analyzing Sarkar's Construction}
\label{sec:sarkar}
The \emph{Voronoi cell} around a node $a \in T$ consists of points $x \in \mathbb{H}_2$ such that $d_H(f(a),x) \leq d_H(f(b),x)$ for all $b \in T$ distinct from $a$. That is, the cell around $a$ includes all points closer to $f(a)$ than to any other embedded node of the tree. Sarkar's construction produces Delauney embeddings: embeddings where the Voronoi cells for points $a$ and $b$ touch only if $a$ and $b$ are neighbors in $T$. Thus this embedding will preserve neighborhoods.

A key technical idea exploited by \citet{sarkar} is to scale all the
edges by a factor $\tau$ before embedding. We can then recover the original distances
by dividing by $\tau$. This transformation exploits the fact that
hyperbolic space is not {\em scale invariant}.
Sarkar's construction always captures neighbors perfectly, but Figure~\ref{fig:geod} implies that increasing the scale preserves the distances between farther nodes better.
Indeed, if one sets
$\tau
= \frac{1+\varepsilon}{\varepsilon}\left(2\log \frac{\operatorname{deg}_{\max}}{\pi
/2}\right)$, then the worst-case distortion $D$ of the resulting embedding is no more than
$1+\varepsilon$. For trees, Sarkar's construction has arbitrarily high
fidelity. However, this comes at a cost: the scaling $\tau$ affects
the bits of precision required. In fact, we will show that the
precision scales logarithmically with the degree of the tree---but linearly with the maximum path length. We use
this to better understand the situations in which hyperbolic
embeddings obtain high quality.

%Algorithm~\ref{alg:sarkar} produces edges scaled by $\nu$. That is, a unit distance has been scaled to $\nu$, and we can recover the original distances by dividing by $\nu$. Choosing $\nu$ correctly allows us to bound the distortion $d_{wc}$ of the embedding to $1+\varepsilon$ for any $\varepsilon > 0$.

%We call this Sarkar's condition.

%We build on the $\mathbb{H}_2$ construction from \citet{sarkar}. This approach produces Delauney embeddings, i.e., embeddings $f$ where the Voronoi cells for points $f(x),f(y)$ \footnote{The \emph{Voronoi cell} around $f(x)$ consists of points $\alpha \in \mathbb{H}_2$ such that $d_H(f(x),\alpha) \leq d_H(f(y),\alpha)$ for all $y \in T$ distinct from $x$. That is, the cell around $f(x)$ includes all points closer to $f(x)$ than any other embedded vertex of the tree.} touch only if $x,y$ are neighbors in $T$. The basic idea is to embed the children of $x$ so that each child is placed inside a disjoint cone emanating from $x$. Moreover, the cones rooted at child $y$ lie inside the cone rooted at $x$ containing $y$, so that Voronoi cells around nodes in different subtrees cannot touch. Full details on this approach are found in \citet{sarkar}.

How many bits of precision do we need to represent points in
$\mathbb{H}_2$? If $x \in \mathbb{H}_2$, then $\|x \| < 1$, so we need
sufficiently many bits so that $1 - \|x\|$ will not be rounded to zero. This requires
roughly $-\log (1-\|x\|) = \log \frac{1}{1-\|x\|}$ bits.  Say we are
embedding two points $x,y$ at distance $d$. As described in the
background, there is an isometric reflection that takes a pair of points $(x,y)$
in $\mathbb{H}_2$ to $(0,z)$ while preserving their distance, so
without loss of generality we have that
\[ d = d_H(x, y) = d_H(0,z) = \acosh \left(1+2\frac{\|z\|^2}{1-\|z\|^2} \right). \]
Rearranging the terms, we have \[\frac{\cosh(d)+1}{2} = \frac{1}{1-\|z\|^2} \ge \frac{1/2}{1-\|z\|}.\] Thus, the number of bits we want so that $1 - \|z\|$ will not be rounded to zero is $\log ( \cosh(d)+1)$. Since $\cosh(d) = (\exp(d)+\exp(-d))/2$, this is roughly $d$ bits.
That is, in hyperbolic space, we need about $d$ bits to express distances of $d$ (rather than $\log d$ as we would in Euclidean space).%
\footnote{Although it is particularly easy to bound precision in the Poincar{\'e} model, this fact holds generally for hyperbolic space independent of model. See Appendix~\ref{app:CombinatorialProofs} for a general lower bound argument.}
This result will be of use below.

Now we consider the largest distance in the embeddings produced by Algorithm~\ref{alg:sarkar}. If the longest path in the tree is $\ell$, and each edge has length $\tau = \frac{1}{\varepsilon}\left(2\log \frac{\operatorname{deg}_{\text{max}}}{\pi /2}\right)$, the largest distance is $O(\frac{\ell}{\varepsilon}\log \operatorname{deg}_{\text{max}})$, and we require this number of bits for the representation.

We interpret this expression. Note that $\operatorname{deg}_{\max}$ is inside the $\log$ term, so that a bushy tree is not penalized much in precision. On the other hand, the longest path length $\ell$ is not, so that hyperbolic embeddings struggle with long paths. 
Moreover, by selecting an explicit graph, we derive a matching lower
bound, concluding that to achieve a distortion $\varepsilon$, any
construction requires $\Omega\left(\frac{\ell}{\varepsilon} \log (\text{deg}_{\max}) \right)$
bits, which matches the upper bound of the combinatorial
construction. The argument follows from selecting a graph consisting
of $m(\text{deg}_{\max}+1)$ nodes in a tree with a single root and $\text{deg}_{\max}$ chains each of length $m$. The
proof of this result is described in Appendix~\ref{app:CombinatorialProofs}.

%% \begin{figure}
%% \centering
%% \includegraphics[width=0.50\textwidth]{figures/chain2.pdf}
%% \caption{Graphs with long chains are difficult for hyperbolic embeddings (left) while short bushy trees are easy (right).}
%% \label{fig:chains}
%% \end{figure}

%Our implementation of the construction follows \cite{Sarkar} closely, with simple modifications required by the fact that in \cite{Sarkar}, the construction is performed in (abstract) hyperbolic space, while we work with the Poincar\'{e} models. 

\subsection{Improving the Construction}

%\begin{table}[tb]
%\centering
%\begin{tabular}{|l||c|c|c|c|c|}
%\hline 
%Dataset & Nodes & $d_{\max}$ &  $r=2$ & $r=3$ & $r=5$ \\    \hline    \hline
%Bal. Tree     & 40    &4  & 25 & 13  & 13 \\ \hline
%Phy. Tree & 344  &16 & 425 & 142 & 107\\ \hline
%\end{tabular}
%\caption{Precision upper bound required for combinatorial construction at $\varepsilon=1.0$ tolerance for two trees described in Section~\ref{sec:experiments}.}
%\label{table:bitcost}
%\end{table}

Our next contribution is a generalization of the construction from the disk $\mathbb{H}_2$ to the ball $\mathbb{H}_r$. Our construction follows the same line as Algorithm~\ref{alg:sarkar}, but since we have $r$ dimensions, the step where we place children spaced out on a circle around their parent now uses a hypersphere.

Spacing out points on the hypersphere is a classic problem known as \emph{spherical coding} \cite{Spheres}. As we shall see, the number of children that we can place for a particular angle grows with the dimension. Since the required scaling factor $\tau$ gets larger as the angle decreases, we can reduce $\tau$ for a particular embedding by increasing the dimension. Note that increasing the dimension helps with bushy trees (large $\operatorname{deg}_{\max}$), but has limited effect on tall trees with small $\operatorname{deg}_{\max}$. We show

%there are $r-1$ angles $\theta_1, \ldots, \theta_{r-1}$. We divide the angles into $k$ parts, allowing us to place $\Theta(k^{r-1})$ children around any node for $k\geq 2$. Since we need $k^{r-1} \geq \operatorname{deg}_{\max}$, we can ultimately reduce the precision linearly in $r$ for $r$ up to $\leq (\log \operatorname{deg}_{\max})+1$. 

\begin{proposition} The generalized $\mathbb{H}_r$ combinatorial construction has distortion at most $1+\varepsilon$ and requires at most $O(\frac{1}{\varepsilon}\frac{\ell}{r} \log \operatorname{deg}_{\max})$ bits to represent a node component for $r \leq (\log \operatorname{deg}_{\max})+1$, and $O(\frac{1 }{\varepsilon}\ell)$ bits for $r > (\log \operatorname{deg}_{\max})+1$. 
\end{proposition}

The algorithm for the generalized $\mathbb{H}_r$ combinatorial construction replaces Step~5 in Algorithm~\ref{alg:sarkar} with a node placement step based on ideas from coding theory. The children are placed at the vertices of a hypercube inscribed into the unit hypersphere (and afterwards scaled by $\tau$). Each component of a hypercube vertex has the form $\frac{\pm 1}{\sqrt{r}}$. We index these points using binary sequences ${a} \in \{0,1\}^r$ in the following way:

\[{x}_{ a} = \left( \frac{(-1)^{a_1}}{\sqrt{r}}, \frac{(-1)^{a_2}}{\sqrt{r}} , \ldots, \frac{(-1)^{a_r}}{\sqrt{r}} \right).\]

We can space out the children by controlling the distances between the children. This is done in turn by selecting a set of binary sequences with a prescribed minimum Hamming distance---a binary error-correcting code---and placing the children at the resulting hypercube vertices. We provide more details on this technique and our choice of code in the appendix.

\subsection{Embedding into Trees}
%Since embedding trees into the Poincar\'{e} can be performed with distortion as low as we desire, it remains to consider the distortion in the first stage. 
\begin{figure}
\centering
\includegraphics[width=0.3\textwidth]{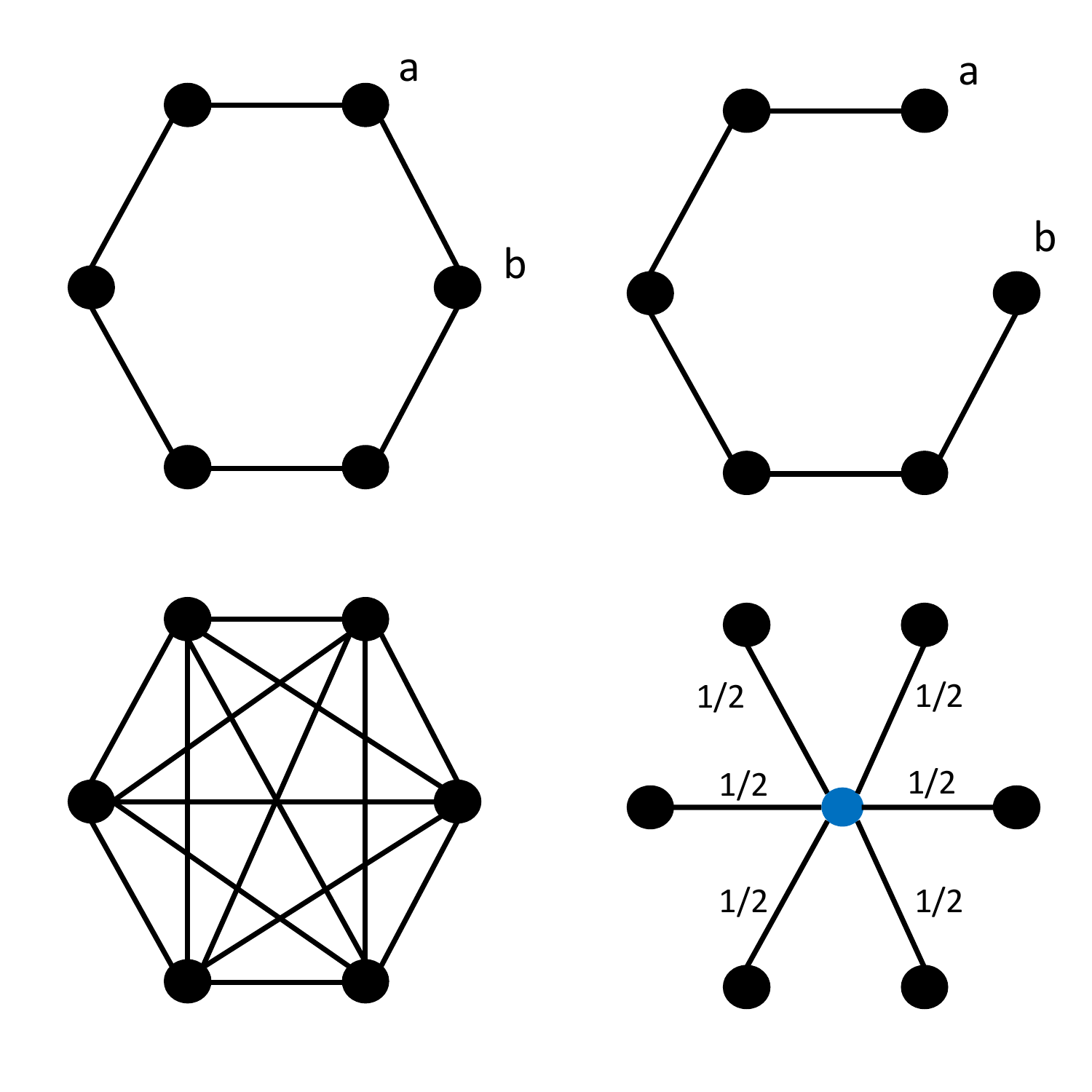}
\caption{Top. Cycles are a challenge for tree embeddings: $d_G(a,b)$ goes from $1$ to $5$. Bottom. Steiner nodes can help: adding a node (blue) and weighting edges maintains the pairwise distances.}
\label{fig:steiner}
\end{figure}
We revisit the first step of the construction: embedding graphs into
trees. There are fundamental limits to how well graphs can be embedded
into trees; in general, breaking long cycles inevitably adds
distortion, as shown in Figure~\ref{fig:steiner}. We are inspired by a
measure of this limit, the \emph{$\delta$-4 points condition} introduced
in \citet{Abraham}. A graph on $n$ nodes that satisfies the $\delta$-4
points condition has distortion at most $(1+\delta)^{c_1 \log n}$ for
some constant $c_1$. This result enables our end-to-end embedding to
achieve a distortion of at most \[D(f) \leq (1+\delta)^{c_1 \log n}(1
+ \varepsilon).\]

The result in \citet{Abraham} builds a tree with \emph{Steiner} nodes. These additional nodes can help control the distances in the resulting tree.
\begin{example} \label{ex:steiner}
Embed a complete graph on $\{1,2,\ldots, n\}$ into a tree. The tree will have a central node, say 1, w.l.o.g., connected to every other node; the shortest paths between pairs of nodes in $\{2,\ldots,n\}$ go from distance 1 in the graph to distance 2 in the tree. However, we can introduce a Steiner node $n+1$ and connect it to all of the nodes, with edge weights of $\frac{1}{2}$. This is shown in Figure~\ref{fig:steiner}. The distance between any pair of nodes in $\{1,\ldots,n\}$ remains 1.\end{example}

Note that introducing Steiner nodes can produce a weighted tree, but Algorithm~\ref{alg:sarkar} readily extends to the case of weighted trees by modifying Step~5.
We propose using the Steiner tree algorithm in \citet{Abraham} (used to achieve the distortion bound) for real embeddings, and we rely on it for our experiments in Section~\ref{sec:experiments}. In summary, 
the key takeaways of our analysis in this section are:
\begin{itemize}
\item
There is a fundamental tension between precision and quality in
hyperbolic embeddings.

\item Hyperbolic embeddings have an exponential advantage in space compared to Euclidean embeddings for short, bushy hierarchies, but will have less of an advantage
for graphs that contain long paths.

\item Choosing an appropriate scaling factor $\tau$ is critical for quality.
Later, we will propose to learn this scale factor automatically for computing embeddings in PyTorch.
\item Steiner nodes can help improve embeddings of graphs.
\end{itemize}

\section{Hyperbolic Multidimensional Scaling}
\label{sec:MDS}
In this section, we explore a fundamental and more general question than we did in the previous section: if we are given the pairwise distances arising from a set of points in hyperbolic space, can we recover the points? The equivalent problem for Euclidean distances is solved with multidimensional scaling (MDS). The goal of this section is to analyze the \emph{hyperbolic MDS} (h-MDS) problem. We describe and overcome the additional technical challenges imposed by hyperbolic distances, and show that exact recovery is possible and interpretable.
Afterwards we propose a technique for dimensionality reduction using principal geodesics analysis (PGA) that provides optimization guarantees.
In particular, this addresses the shortcomings of h-MDS when recovering points that do not exactly lie on a hyperbolic manifold.
%proceed to analyze perturbations for h-MDS (i.e., recovery from noisy distances), mirroring the analysis of MDS robustness.

\subsection{Exact Hyperbolic MDS}
\label{sec:exactmds}

Suppose that there is a set of
hyperbolic points $x_1,\dots, x_n \in \mathbb{H}_r$, embedded in the Poincar{\'e} ball and written $X \in
\mathbb{R}^{n \times r}$ in matrix form.
We observe all the pairwise distances $d_{i,j} = d_H(x_i, x_j)$, but do not observe $X$:
our goal is use the observed $d_{i,j}$'s to recover $X$ (or some other set of points with the same pairwise distances $d_{i,j}$).

The MDS algorithm in the Euclidean setting makes an important
\emph{centering}%
\footnote{We say that points are centered at a particular mean
  if this mean is at $0$. The act of centering refers to applying an isometry
  that makes the mean of the points $0$.}
assumption.
That is it assumes the points have mean $0$, and it turns out that if an exact
embedding for the distances exists, it can be recovered from a matrix factorization.
In other words, Euclidean MDS always recovers a centered embedding.

In hyperbolic space, the same algorithm does not work, but we show that it is possible to find an embedding centered at a different mean. 
More precisely, we introduce a new mean which we call the \emph{pseudo-Euclidean mean}, that behaves like the Euclidean mean in that it enables recovery through matrix factorization.
Once the points are recovered in hyperbolic space, they can be recentered around a more canonical mean by translating it to the origin.

Algorithm~\ref{alg:new_hmds} is our complete algorithm, and for the remainder of
this section we will describe how and why it works.
We first describe the \emph{hyperboloid model}, an alternate but equivalent model of hyperbolic geometry in which h-MDS is simpler. Of course, we can easily convert between the hyperboloid model and the Poincar\'{e} ball model we have used thus far.
Next, we show how to reduce the problem to a standard PCA problem, which recovers an embedding centered at the points' pseudo-Euclidean mean.
Finally, we discuss the meaning and implications of centering and prove that the algorithm preserves submanifolds as well---that is, if there is an exact embedding in $k < r$ dimensions centered at their canonical mean,
then our algorithm will recover them.

\paragraph*{The hyperboloid model}
Define $Q$ to be the diagonal matrix in $\R^{r+1}$ where $Q_{00} = 1$ and $Q_{ii} = -1$ for $i > 0$.
For a vector $x \in \R^{r+1}$, $x^TQx$ is called the \emph{Minkowski quadratic form}.
The hyperboloid model is defined as
\[
  \mathbb{M}_r = \left\{ x \in \R^{r+1} \middle| x^T Q x = 1 \land x_0 > 0 \right\}.
\]
This manifold is endowed with a distance measure
\[
  d_H(x, y) = \acosh(x^T Q y).
\]
As a notational convenience, for a point $x \in \mathbb{M}_r$ we will let $x_0$ denote $0$th coordinate $e_0^T x$, and let $\vec x \in \R^r$ denote the rest of the coordinates.
Notice that $x_0$ is just a function of $\vec x$ (in fact, $x_0 = \sqrt{1 + \| \vec{x} \|^2}$), and so we can equivalently consider just $\vec x$ as being a member of a model of hyperbolic space: this model is sometimes known as the Gans model.
With this notation, the Minkowski quadratic form can be simplified to $x^T Q y = x_0 y_0 - \vec{x}^T \vec{y}$.

\paragraph*{A new mean}
We introduce the new mean that we will use.
Given points $x_1, x_2, \ldots, x_n \in \mathbb{M}_r$ in hyperbolic space,
define a variance term
\[
  \Psi(z; x_1, x_2, \ldots, x_n)
  =
  \sum_{i=1}^n \sinh^2(d_H(x_i, z)).
\]
Using this, we define a \emph{pseudo-Euclidean mean} to be any local minimum of this expression.
% This is a type of \emph{Karcher mean} in hyperbolic space.%
% \footnote{A Karcher mean is a local minimum of...}
% \[
%   A(x_1, \ldots, x_n)
%   =
%   \arg \min_{z \in \mathbb{M}_r} \Psi(z; x_1, \ldots, x_n).
% \]
Notice that this average is independent of the model of hyperbolic space that we are using, since it only is defined in terms of the hyperbolic distance function $d_H$.

\begin{lemma}
  \label{lmm:pe-centered}
  Define the matrix $X \in \R^{n \times r}$ such that $X^T e_i = \vec{x}_i$ and the vector $u \in \R^n$ such that $u_i = x_{0,i}$.
  Then
  \begin{align*}
    \left. \nabla_{\vec{z}} \Psi(z; x_1, x_2, \ldots, x_n) \right|_{\vec{z} = 0}
    =
    -2 \sum_{i=1}^n x_{0,i} \vec{x}_i
    =
    -2 X^T u.
  \end{align*}
\end{lemma}
This means that $0$ is a pseudo-Euclidean mean if and only if $0 = X^T u$.
Call some hyperbolic points $x_1, \ldots, x_n$ \emph{pseudo-Euclidean centered} if their average is $0$ in this sense: i.e. if $X^T u = 0$.
We can always center a set of points without affecting their pairwise distances by simply finding their average, and then sending it to $0$ through an isometry.

\paragraph*{Recovery via matrix factorization}
Suppose that there exist points $x_1, x_2, \ldots, x_n \in \mathbb{M}_r$
for which we observe their pairwise distances $d_H(x_i, x_j)$.
From these, we can compute the matrix $Y$ such that
\begin{equation}
  \label{eq:hmds-Y}
  Y_{i,j} = \cosh\left( d_H(x_i, x_j) \right) = x_i^T Q x_j = x_{0,i} x_{0,j} - \vec{x_i}^T \vec{x_j}.
\end{equation}
Furthermore, defining $X$ and $u$ as in Lemma~\ref{lmm:pe-centered},
then we can write $Y$ in matrix form as
\begin{equation}
  \label{eq:hmds-Y2}
  Y = u u^T - X X^T.
\end{equation}
Without loss of generality, we can suppose that the points we are trying to recover, $x_1, \ldots, x_n$, are centered at their pseudo-Euclidean mean, so that $X^T u = 0$ by Lemma~\ref{lmm:pe-centered}.

This implies that $u$ is an eigenvector of $Y$ with positive eigenvalue, and the rest of $Y$'s eigenvalues are negative.
Therefore an eigendecomposition of $Y$ will find $u,\hat{X}$ such that $Y = u u^T - \hat{X} \hat{X}^T$,
i.e. it will directly recover $X$ up to rotation.

In fact, running PCA on $-Y = X^T X - u u^T$ to find the $n$ most significant non-negative eigenvectors will recover $X$ up to rotation,
and then $u$ can be found by leveraging the fact that $x_0 = \sqrt{1 + \| \vec{x} \|^2}$.

This leads to Algorithm~\ref{alg:new_hmds}, with optional post-processing steps for converting the embedding to the Poincar{\'e} ball model and for re-centering the points.
% First, this algorithm returns an embedding in the Gans model; they can be converted to the Poincar{\'e} disk model with a simple projection.
% Second, once we've recovered the points centered at their pseudo-Euclidean mean, we can recover the points centered at any other mean by reflecting it onto the origin.

\paragraph*{A word on centering}
The MDS algorithm in Euclidean geometry returns points centered at their \emph{Karcher mean} $z$, which is a point minimizing $\sum d^2(z, x_i)$ (where $d$ is the distance metric).
The Karcher center is particularly useful for interpreting dimensionality reduction; for example, we use the analogous hyperbolic Karcher mean to perform PGA in Section~\ref{sec:PGA}.

Although Algorithm~\ref{alg:new_hmds} returns points centered at their pseudo-Euclidean mean instead of their Karcher mean, they can be easily recentered
by finding their Karcher mean and reflecting it onto the origin. 
Furthermore, we show that Algorithm~\ref{alg:new_hmds} \emph{preserves the dimension of the embedding}.
More precisely, we prove Lemma~\ref{lmm:hmds-centering} in Appendix~\ref{sec:mds-proof}.
\begin{lemma}
  \label{lmm:hmds-centering}
  If a set of points lie in a dimension-$k$ geodesic submanifold, then both their Karcher mean and their pseudo-Euclidean mean lie in the same submanifold.
\end{lemma}
This implies that centering with the pseudo-Euclidean mean preserves geodesic submanifolds:
If it is possible to embed distances in a dimension-$k$ geodesic submanifold centered and rooted at a Karcher mean, then it is also possible to embed the distances in a dimension-$k$ submanifold centered and rooted at a pseudo-Euclidean mean, and vice versa.

\begin{algorithm}[t]
% \caption{h-MDS}
\caption{ }
\begin{algorithmic}[1]
\STATE {\bfseries Input: Distance matrix $d_{i,j}$ and rank $r$}
\STATE Compute scaled distance matrix $Y_{i,j} = \cosh(d_{i,j})$
\STATE $X \rightarrow \text{PCA}(-Y,r)$
\STATE Project $X$ from hyperboloid model to Poincar\'{e} model: $x \to \frac{x}{1 + \sqrt{1 + \|x\|^2}}$
\STATE If desired, center $X$ at a different mean (e.g. the Karcher mean)
\STATE \textbf{return} $X$
\end{algorithmic}
\label{alg:new_hmds}
\end{algorithm}

%%% Local Variables:
%%% mode: latex
%%% TeX-master: "hyperbolic_arxiv"
%%% End:

\subsection{Reducing Dimensionality with PGA}
\label{sec:PGA}
%{\color{red} This will get longer, with a result}.
Sometimes we are given a
high-rank embedding (resulting from h-MDS, for example), and wish to
find a lower-rank version.
In Euclidean space, one can get the optimal
lower rank embedding by simply discarding components. However, this
may not be the case in hyperbolic space.
Motivated by this, we study dimensionality reduction in hyperbolic space.

As hyperbolic space does not have a linear subspace structure like
Euclidean space, we need to define what we mean by lower-dimensional.
We follow Principal Geodesic Analysis~\cite{PGA}, \cite{GPCA}. Consider an initial
embedding with points $x_1,\dots,x_n \in \mathbb{H}_2$ and let $d_{H}
: \mathbb{H}_2 \times \mathbb{H}_2 \to \mathbb{R}_{+}$ be the hyperbolic distance.
Suppose we want to map this embedding onto a one-dimensional subspace. (Note that
we are considering a two-dimensional embedding and one-dimensional subspace
here for simplicity, and these
results immediately extend to higher dimensions.) In this case, the goal of PGA
is to find a geodesic $\gamma :
[0,1] \to \mathbb{H}_2$ that passes through the mean of the points and that minimizes the squared error (or variance):
\[ f(\gamma) = \sum_{i=1}^n \min_{t \in [0,1]} d_{H}(\gamma(t),x_i)^2 .\]
This expression can be simplified significantly and reduced to a
minimization in Euclidean space.  First, we find the mean of the
points, the point $\bar x$ which minimizes $\sum_{i=1}^n d_{H}(\bar
x, x_i)^2$; this definition in terms of distances generalizes the mean in Euclidean space.\footnote{As we noted earlier, considering the distances
without squares leads to a non-continuously-differentiable
formulation.}  Next, we reflect all the points $x_i$ so that their
mean is $0$ in the Poincar{\'e} disk model; we can do this using a
circle inversion that maps $\bar x$ onto $0$.
In the Poincar{\'e} disk model, a geodesic through
the origin is a Euclidean line, and the action of the reflection across
this line is the same in both Euclidean and hyperbolic space. Coupled
with the fact that reflections are isometric, if $\gamma$ is a line
through $0$ and $R_\gamma$ is the reflection across $\gamma$, we have
\[
  d_H(\gamma, x) = \min_{t \in [0,1]} d_H(\gamma(t), x) = \frac{1}{2} d_H(R_l x, x).
\]
Combining this with the Euclidean reflection formula and the hyperbolic metric produces
\[
  f(\gamma) = \frac{1}{4} \sum_{i=1}^n \acosh^2\left( 1 + \frac{ 8 d_{E}(\gamma,x_i)^2 }{(1 - \| x_i \|^2)^2} \right),
\]
in which $d_{E}$ is the Euclidean distance from a point to a line. If
we define $w_i = \sqrt{8} x_i / (1 - \| x_i \|^2)$ this reduces to the simplified expression
\[
  f(\gamma) = \frac{1}{4} \sum_{i=1}^n \acosh^2\left( 1 + d_{E}(\gamma,w_i)^2 \right).
\]
  
Notice that \emph{the loss function is not convex}. We observe that
there can be multiple local minima that are attractive and stable, in
contrast to PCA.  Figure~\ref{fig:pga} illustrates this nonconvexity
on a simple dataset in $\mathbb{H}_2$ with only four examples.  This
makes globally optimizing the objective difficult.
\begin{figure}
\centering
\resizebox{0.48\textwidth}{!}{\large\input{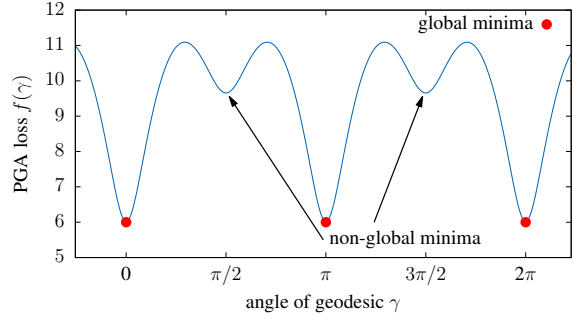}}
\caption{The PGA objective of an example task where the input dataset in the Poincar{\'e} disk is $x_1 = (0.8,0)$, $x_2 = (-0.8,0)$, $x_3 = (0,0.7)$ and $x_4 = (0,-0.7)$. Note the presence of non-optimal local minima, unlike PCA.}
\label{fig:pga}
\end{figure}

Nevertheless, there will always be a region $\Omega$ containing a
global optimum $\gamma^*$ that is convex and admits an efficient
projection, and where $f$ is convex when restricted to $\Omega$. Thus
it is possible to build a gradient descent-based algorithm to recover
lower-dimensional subspaces: for example, we built a simple optimizer
in PyTorch.  We also give a sufficient condition on the data for $f$ above
to be convex.
\begin{lemma}
For hyperbolic PGA if for all $i$,
\[
  \acosh^2\left( 1 + d_{E}(\gamma,w_i)^2 \right) < \min\left(1, \frac{1}{3} \| w_i \|^2 \right)
\]
then $f$ is locally convex at $\gamma$.
\label{lemma:pga}
\end{lemma}
As a result, if we initialize in and optimize over a region that
contains $\gamma^*$ and where the condition of Lemma~\ref{lemma:pga}
holds, then gradient descent will be guaranteed to converge to
$\gamma^*$. We can turn this result around and read it as a recovery
result: if the noise is bounded in this regime, then we are able to
provably recover the correct low-dimensional embedding.

\section{Experiments}
\label{sec:experiments}

\begin{table}\centering
\begin{tabular}{|c||c|c|c|}\hline
  Dataset & Nodes & Edges & Comment\\
  \hline Bal. Tree & 40 & 39 & Tree\\
  Phy. Tree & 344 & 343 & Tree\\
  \hline\hline CS PhDs & 1025 & 1043 & Tree-like\\
  WordNet & 74374 & 75834 & Tree-like\\
  \hline\hline Diseases & 516 & 1188 & Dense\\
 Gr-QC & 4158 &  13428& Dense\\
  \hline
\end{tabular}
\caption{Datasets Statistics.}
\end{table}

\begin{table}[h]
\centering
\begin{tabular}{|l||c|c|c|} \hline
Dataset     	          &  C-$\mathbb{H}_2$ &  FB $\mathbb{H}_5$ & FB $\mathbb{H}_{200}$                  \\ \hline\hline
WordNet & 0.989 & 0.823* & 0.87*\\
  \hline
\end{tabular}
\caption{MAP measure for WordNet embedding compared to values in \citet{fb}. Closer to 1 is better.}
\label{table:wordnet_results}
\end{table}

\begin{table}[h]
\centering
\begin{tabular}{|l||c|c||c|c|c|c|c|} \hline
  Dataset     	          &  C-$\mathbb{H}_2$ &  FB $\mathbb{H}_2$ & h-MDS & PyTorch & PWS & PCA & FB                  \\ \hline\hline
  Bal. Tree         & {\bf 0.013}         &      0.425               &    {\bf 0.077}     & 0.034 & 0.020   &  0.496    & 0.236 \\ 
  Phy. Tree          & {\bf 0.006}             &     0.832               &  {\bf   0.039}    & 0.237 & 0.092 &   0.746        &       0.583     \\
\hline \hline
CS PhDs              & {\bf 0.286}    &   0.542                           &  {\bf 0.149}  & 0.298 & 0.187   &  0.708  & 0.336   \\ 
%WordNet              & {\bf 0.054}           &     0.793                &               &  & 0.503\\
\hline\hline 
Diseases              & {\bf 0.147}    &    0.410                          &  0.111  & {\bf 0.080} &  0.108   &    0.595 &  0.764              \\ \hline 
%Gr-QC               & 0.603    &   {\bf 0.387}                               & 0.530    &   0.546 &   0.713\\ \hline 
Gr-QC               & {\bf 0.354}    &   -                               & 0.530  & {\bf 0.125} & 0.134 &   0.546 &   -\\ \hline

\end{tabular}
\caption{Distortion measures using combinatorial and h-MDS techniques, compared against PCA and results from \citet{fb}. Closer to 0 is better.}
\label{table:distortion_results}
\end{table}

\begin{table}[h]
\centering
\begin{tabular}{|l||c|c||c|c|c|c|c|} \hline
  Dataset     	          &  C-$\mathbb{H}_2$ &  FB $\mathbb{H}_2$ & h-MDS & PyTorch & PWS & PCA & FB                  \\ \hline\hline
  Bal. Tree          & {\bf 1.0}              &    0.846    &    {\bf 1.0}    & {\bf 1.0} & {\bf 1.0}       &  {\bf 1.0}           & 0.859     \\ 
Phy. Tree          &      {\bf 1.0}                &    0.718  &   0.675    & 0.951 & 0.998          &    {\bf 1.0}          &       0.811     \\
\hline
\hline
CS PhDs              & {\bf 0.991}             &  0.567          &  0.463    &     0.799 & {\bf 0.945} &              0.541 & 0.78  \\ 
%WordNet              &     {\bf 0.989}            &  0.823$^*$                          &     &            &   0.870 $^*$\\ \hline
\hline\hline
Diseases              &      {\bf 0.822}             & 0.788          &     0.949    &  0.995  &  0.897      &         {\bf 0.999}      &    0.934\\
\hline
%Gr-QC                          &    0.759      &  0.635               &    0.710   &  0.738  &  {\bf 0.999}$^*$  \\ \hline 
Gr-QC                          &    0.696      & -               &    0.710  & 0.733 & 0.504 &  0.738  &  {\bf 0.999}$^*$  \\ \hline 
\end{tabular}
\caption{MAP measures using combinatorial and h-MDS techniques, compared against PCA. %Note results with asterix are from \citet{fb}.
Closer to 1 is better.}
\label{table:map_results}
\end{table}

We evaluate the proposed approaches and compare against
existing methods. We hypothesize that for tree-like data, the
combinatorial construction offers the best performance. For general
data, we expect h-MDS to produce the lowest distortion, while it may
have low MAP due to precision limitations. We anticipate that
dimension is a critical factor (outside of the combinatorial
construction). Finally, we expect that the MAP of h-MDS techniques can
be improved by learning the correct scale and weighting the loss
function as suggested in earlier sections. In the Appendix, we report
on additional datasets, parameters found by the combinatorial
construction, and the effect of hyperparameters.%, and on reducing 
%the precision required by the combinatorial construction by increasing the dimension; 
%as we observed in Section~\ref{sec:combinatorial}, we can reduce the scaling by
%a factor depending on $\operatorname{deg}_{\max}$, but the $\frac{1+\epsilon}{\epsilon}$ term is
%significant.
%For $\varepsilon=1.0$, however, we can store components in 64 bits.

\begin{figure}
\centering
\includegraphics[width=0.4\textwidth]{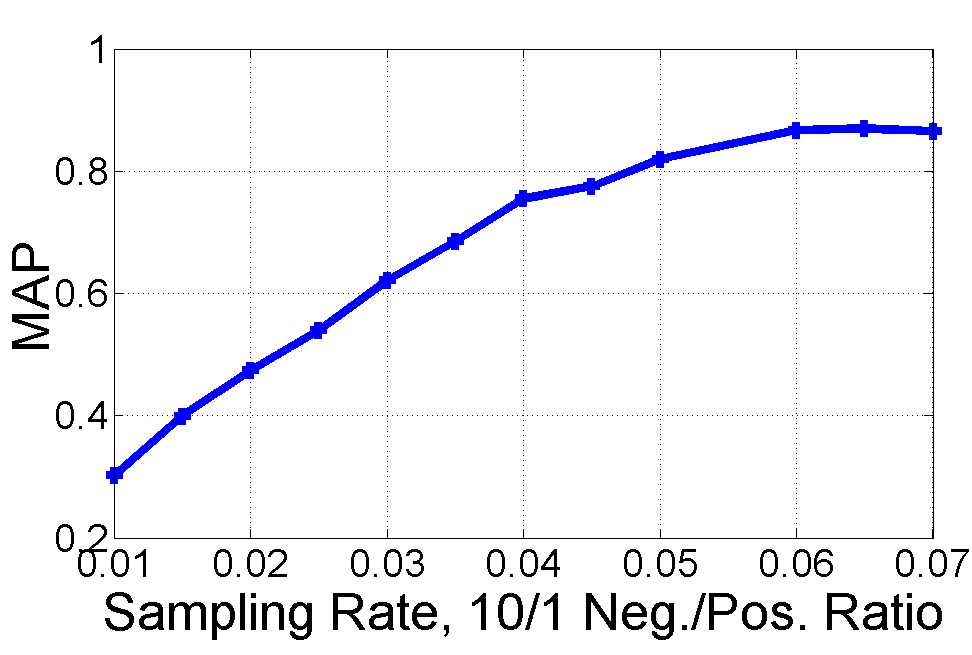}
\includegraphics[width=0.4\textwidth]{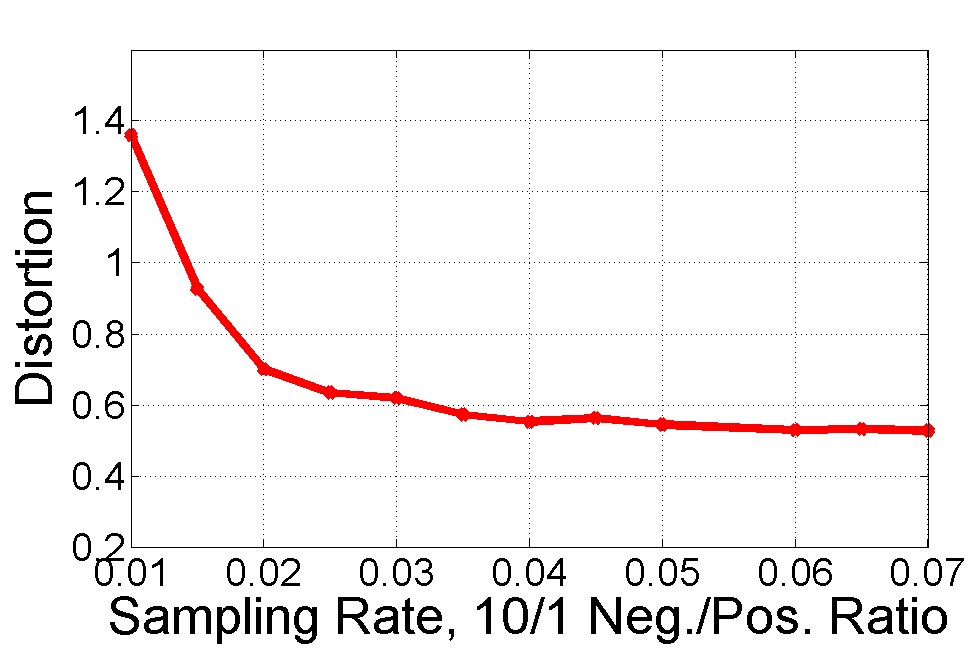}
\caption{Learning from incomplete information. The distance matrix is
  sampled, completed, and embedded.}
\label{fig:sampling}
\end{figure}

%\yell{We we are able to deal with
%  incomplete data. Vary amount of data}

\begin{table}[]
\centering
\begin{tabular}{|r||c|c|c|}
\hline 
Rank  & No Scale  &  Learned Scale  &  Exp. Weighting \\    \hline    \hline
50 & 0.481 & 0.508 & {\bf 0.775} \\  \hline
100 & 0.688 & 0.681 & {\bf 0.882}\\ \hline
200 & 0.894 & 0.907 & {\bf 0.963}\\ \hline
\end{tabular}
\caption{Ph.D. dataset. Improved MAP performance of PyTorch implementation using a modified PGA-like loss function.}
\label{table:pytorch}
\end{table}

\paragraph*{Datasets}
We consider trees, tree-like hierarchies, and graphs that are not
tree-like. First, we consider hierarchies that form trees:
fully-balanced trees along with phylogenetic trees expressing genetic
heritage (of mosses growing in urban environments \cite{phylo-tree}, available from \cite{treebase}). Similarly, we used hierarchies that are nearly tree-like:
WordNet hypernym (the largest connected component from~\citet{fb}) and
a graph of Ph.D. advisor-advisee relationships \cite{de2011exploratory}. Also included are
%several
datasets that vary in their tree nearness, such as biological sets
involving disease relationships \cite{goh2007human} and protein interactions in yeast
bacteria \cite{jeong2001lethality}, both available from \cite{nr}. We also include the collaboration network formed by authorship relations for papers submitted to the general relativity and quantum cosmology (Gr-QC) arXiv \cite{grqc}.
%, and a scientific authorship network involving papers
% from the Gr-QC arXiv that were used in prior work.

\paragraph*{Approaches} Combinatorial
embeddings into $\mathbb{H}_2$ are done using Steiner trees generated
from a randomly selected root for the $\varepsilon=0.1$ precision
setting; others are considered in the Appendix. We performed h-MDS in
floating point precision.
We also include results for our PyTorch implementation of an SGD-based algorithm (described later),
as well as a warm start version initialized with the high-dimensional combinatorial construction.
We compare against classical MDS (i.e.,
PCA), and the optimization-based approach~\citet{fb}, which we call
FB. The experiments for h-MDS, PyTorch SGD, PCA, and FB used dimensions of
2,5,10,50,100,200; we recorded the best resulting MAP and
distortion. Due to the large scale, we did not replicate the best FB
numbers on large graphs (i.e., Gr-QC and WordNet).
%(e.g., WordNet).
As a result, we
report their best published MAP numbers for comparison (their work does not report distortion). These entries are marked with an asterisk. For the WordNet
graph we use a random BFS tree rather than a Steiner tree. Moreover, for comparison against FB (which computes the transitive closure), we can use a weighted version of the graph that captures the ancestor relationships. The full details are in Appendix.

\paragraph*{Quality}
In Table~\ref{table:distortion_results}, we report the distortion. As
expected, when the graph is a tree or tree-like the combinatorial
construction has exceedingly low distortion. Because h-MDS is meant to recover
points exactly, we hypothesized that h-MDS would offer very low distortion
on these datasets. We confirm this hypothesis: among h-MDS, PCA, and FB, h-MDS
consistently offers the best (lowest) distortion, producing, for
example, a distortion of $0.039$ on the phylogenetic tree dataset. %We
% see that both PCA and FB have larger distortion, as expected.
We observe that the optimization-based approach works quite well for reducing distortion,
and on tree-like datasets it is bolstered by appropriate initialization from the combinatorial construction.

Table~\ref{table:map_results} reports the MAP measure (this is shown in Table~\ref{table:wordnet_results} for WordNet), which is a
local measure. We expect that the combinatorial construction performs
well for tree-like hierarchies. This is indeed the case: on trees and
tree-like graphs, the MAP is close to 1, improving on approaches such
as FB that rely on optimization. On larger graphs like WordNet, our
approach yields a MAP of $0.989$--improving on the FB MAP result of
$0.870$ at 200 dimensions. This is exciting because the combinatorial
approach is deterministic and linear-time. In addition, it suggests
that this refined understanding of hyperbolic embeddings may be used
to improve the quality and runtime state of the art constructions. As
expected, the MAP of the combinatorial construction decreases as the
graphs are less tree-like.
% However, note that the FB construction
% reports a much better MAP on the Gr-QC dataset ($0.999$) than our
% combinatorial construction $0.759$.
Interestingly, h-MDS solved in
floating point indeed struggles with MAP. We separately confirmed that
it is indeed due to precision using a high-precision solver, which
obtains a perfect MAP---but uses 512 bits of precision. It may be
possible to compensate for this with scaling, but we did not explore
this possibility.

\paragraph*{SGD-Based Algorithm}
We also built an SGD-based algorithm implemented in PyTorch. Here the
loss function is equivalent to the PGA loss, and so is continuously
differentiable. We use this to verify two claims:

\textit{Learned Scale.}
In Table~\ref{table:pytorch}, we verify the importance of scaling that
our analysis suggests; our implementation has a simple learned scale
parameter. Moreover, we added an exponential weighting to the
distances in order to penalize long paths, thus improving the local
reconstruction. These techniques indeed improve the MAP; in
particular, the learned scale provides a better MAP at lower rank. We
hope these techniques can be useful in other embedding techniques.

\textit{Incomplete Information.} 
To evaluate our algorithm's ability to deal with incomplete
information, we examine the quality of recovered solutions as we
sample the distance matrix. We set the sampling rate of non-edges to
edges at $10:1$ following~\citet{fb}. We examine the phylogenetic
tree, which is full rank in Euclidean space. In
Figure~\ref{fig:sampling}, we are able to recover a good solution with a
small fraction of the entries for the phylogenetic tree dataset; for
example, we sampled approximately $4\%$ of the graph but provide a MAP
of $0.74$ and distortion of less than $0.6$. Understanding the sample
complexity for this problem is an interesting theoretical question.

\section{Conclusion and Future Work}
\label{sec:conclusion}

Hyperbolic embeddings embed hierarchical information with high
fidelity and few dimensions. We explored the limits of this approach
by describing scalable, high quality algorithms. We hope the
techniques here encourage more follow-on work on the exciting
techniques of \citet{fb, ucl}. As future work, we hope to explore how
hyperbolic embeddings can be most effectively incorporated into downstream
tasks and applications.

\bibliographystyle{plainnat}
\bibliography{hyperbolic_bib}

\begin{thebibliography}{39}
\providecommand{\natexlab}[1]{#1}
\providecommand{\url}[1]{\texttt{#1}}
\expandafter\ifx\csname urlstyle\endcsname\relax
  \providecommand{\doi}[1]{doi: #1}\else
  \providecommand{\doi}{doi: \begingroup \urlstyle{rm}\Url}\fi

\bibitem[Abu-Ata and Dragan(2015)]{Dragan}
M.~Abu-Ata and F.~F. Dragan.
\newblock Metric tree-like structures in real-world networks: an empirical
  study.
\newblock \emph{Networks}, 67\penalty0 (1):\penalty0 49--68, 2015.

\bibitem[Benedetti and Petronio(1992)]{Benedetti}
R.~Benedetti and C.~Petronio.
\newblock \emph{Lectures on Hyperbolic Geometry}.
\newblock Springer, Berlin, Germany, 1992.

\bibitem[Brannan et~al.(2012)Brannan, Esplen, and Gray]{GeometryText}
D.~Brannan, M.~Esplen, and J.~Gray.
\newblock \emph{Geometry}.
\newblock Cambridge University Press, Cambridge, UK, 2012.

\bibitem[Candes and Tao(2010)]{TaoMatrix}
E.~Candes and T.~Tao.
\newblock The power of convex relaxation: Near-optimal matrix completion.
\newblock \emph{IEEE Transactions on Information Theory}, 56\penalty0
  (5):\penalty0 2053--2080, 2010.

\bibitem[Chamberlain et~al.(2017)Chamberlain, Clough, and Deisenroth]{ucl}
B.~P. Chamberlain, J.~R. Clough, and M.~P. Deisenroth.
\newblock Neural embeddings of graphs in hyperbolic space.
\newblock \emph{arXiv preprint, arXiv:1705.10359}, 2017.

\bibitem[Chen et~al.(2012)Chen, Fang, Hu, and Mahoney]{Mahoney}
W.~Chen, W.~Fang, G.~Hu, and M.~W. Mahoney.
\newblock On the hyperbolicity of small-world and tree-like random graphs.
\newblock In \emph{International Symposium on Algorithms and Computation
  (ISAAC) 2012}, pages 278--288, Taipei, Taiwan, 2012.

\bibitem[Conway and Sloane(1999)]{Spheres}
J.~Conway and N.~J.~A. Sloane.
\newblock \emph{Sphere Packings, Lattices and Groups}.
\newblock Springer, New York, NY, 1999.

\bibitem[Cvetkovski and Crovella(2009)]{Crovella}
A.~Cvetkovski and M.~Crovella.
\newblock Hyperbolic embedding and routing for dynamic graphs.
\newblock In \emph{IEEE INFOCOM 2009}, 2009.

\bibitem[Cvetkovski and Crovella(2016)]{cvetkovski2016multidimensional}
Andrej Cvetkovski and Mark Crovella.
\newblock Multidimensional scaling in the poincar{\'e} disk.
\newblock \emph{Applied mathematics \& information sciences}, 2016.

\bibitem[De~Nooy et~al.(2011)De~Nooy, Mrvar, and Batagelj]{de2011exploratory}
W.~De~Nooy, A.~Mrvar, and V.~Batagelj.
\newblock \emph{Exploratory social network analysis with Pajek}, volume~27.
\newblock Cambridge University Press, 2011.

\bibitem[Eppstein and Goodrich(2011)]{Eppstein}
D.~Eppstein and M.~Goodrich.
\newblock Succinct greedy graph drawing in the hyperbolic plane.
\newblock In \emph{Proc. of the International Symposium on Graph Drawing (GD
  2011)}, pages 355--366, Eindhoven, Netherlands, 2011.

\bibitem[Fletcher et~al.(2004)Fletcher, Lu, Pizer, and Joshi]{PGA}
P.~Fletcher, C.~Lu, S.~Pizer, and S.~Joshi.
\newblock Principal geodesic analysis for the study of nonlinear statistics of
  shape.
\newblock \emph{IEEE Transactions on Medical Imaging}, 23\penalty0
  (8):\penalty0 995--1005, 2004.

\bibitem[Ganea et~al.(2018)Ganea, B{\'e}cigneul, and Hofmann]{ganea}
O.~Ganea, G.~B{\'e}cigneul, and T.~Hofmann.
\newblock Hyperbolic entailment cones for learning hierarchical embeddings.
\newblock \emph{arXiv preprint, arXiv:1804.01882}, 2018.

\bibitem[Goh et~al.(2007)Goh, Cusick, Valle, Childs, Vidal, and
  Barab{\'a}si]{goh2007human}
K.~Goh, M.~Cusick, D.~Valle, B.~Childs, M.~Vidal, and A.~Barab{\'a}si.
\newblock The human disease network.
\newblock \emph{Proceedings of the National Academy of Sciences}, 104\penalty0
  (21), 2007.

\bibitem[Gromov(1987)]{Gromov}
M.~Gromov.
\newblock Hyperbolic groups.
\newblock In \emph{Essays in group theory}. Springer, 1987.

\bibitem[Hofbauer et~al.(2016)Hofbauer, Forrest, Hollingsworth, and
  Hart]{phylo-tree}
W.~Hofbauer, L.~Forrest, P.~Hollingsworth, and M.~Hart.
\newblock First insights in the diversity of certain mosses colonising modern
  building surfaces by use of genetic barcoding.
\newblock 2016.

\bibitem[Huckemann et~al.(2010)Huckemann, Hotz, and Munk]{GPCA}
S.~Huckemann, T.~Hotz, and A.~Munk.
\newblock Intrinsic shape analysis: Geodesic pca for riemannian manifolds
  modulo isometric lie group actions.
\newblock \emph{Statistica Sinica}, pages 1--58, 2010.

\bibitem[{I. Abraham et al.}(2007)]{Abraham}
{I. Abraham et al.}
\newblock Reconstructing approximate tree metrics.
\newblock In \emph{Proceedings of the twenty-sixth annual ACM symposium on
  Principles of Distributed Computing (PODC)}, pages 43--52, Portland, Oregon,
  2007.

\bibitem[Jenssen et~al.(2018)Jenssen, Joos, and Perkin]{Jenssen}
M.~Jenssen, F.~Joos, and W.~Perkin.
\newblock On kissing numbers and spherical codes in high dimensions.
\newblock \emph{arXiv preprint, arXiv:1803.02702}, 2018.

\bibitem[Jeong et~al.(2001)Jeong, Mason, Barabasi, and
  Oltvai]{jeong2001lethality}
H.~Jeong, S.P. Mason, A.L. Barabasi, and Z.N. Oltvai.
\newblock Lethality and centrality in protein networks.
\newblock \emph{arXiv preprint cond-mat/0105306}, 2001.

\bibitem[Kleinberg(1999)]{ca-data}
J.~Kleinberg.
\newblock Authoritative sources in a hyperlinked environment.
\newblock 1999.
\newblock URL \url{http://www.cs.cornell.edu/courses/cs685/2002fa/}.

\bibitem[Kleinberg(2007)]{Kleinberg}
R.~Kleinberg.
\newblock Geographic routing using hyperbolic space.
\newblock In \emph{26th IEEE International Conference on Computer
  Communications (ICC)}, pages 1902--1909, 2007.

\bibitem[Krioukov et~al.(2009)Krioukov, Papadopoulos, and
  Bogun{\'a}]{krioukov2009curvature}
D.~Krioukov, F.~Papadopoulos, and A.~Vahdat~M. Bogun{\'a}.
\newblock Curvature and temperature of complex networks.
\newblock \emph{Physical Review E}, 80\penalty0 (3):\penalty0 035101, 2009.

\bibitem[Krioukov et~al.(2010)Krioukov, Papadopoulos, Kitsak, Vahdat, and
  Bogun{\'a}]{krioukov2010hyperbolic}
D.~Krioukov, F.~Papadopoulos, M.~Kitsak, A.~Vahdat, and M.~Bogun{\'a}.
\newblock Hyperbolic geometry of complex networks.
\newblock \emph{Physical Review E}, 82\penalty0 (3):\penalty0 036106, 2010.

\bibitem[Lamping and Rao(1994)]{lamping1994laying}
J.~Lamping and R.~Rao.
\newblock Laying out and visualizing large trees using a hyperbolic space.
\newblock In \emph{Proceedings of the 7th annual ACM symposium on User
  interface software and technology}, pages 13--14. ACM, 1994.

\bibitem[Leskovec et~al.(2007)Leskovec, Kleinberg, and Faloutsos]{grqc}
J.~Leskovec, J.~Kleinberg, and C.~Faloutsos.
\newblock Graph evolution: Densification and shrinking diameters.
\newblock 2007.

\bibitem[Linial et~al.(1995)Linial, London, and Rabinovich]{Lineal}
N.~Linial, E.~London, and Y.~Rabinovich.
\newblock The geometry of graphs and some of its algorithmic applications.
\newblock \emph{Combinatorica}, 15\penalty0 (2):\penalty0 215--245, 1995.

\bibitem[Nickel and Kiela(2017)]{fb}
M.~Nickel and D.~Kiela.
\newblock Poincar\'{e} embeddings for learning hierarchical representations.
\newblock In \emph{Advances in Neural Information Processing Systems 30 (NIPS
  2017)}, Long Beach, CA, 2017.

\bibitem[Pennec(2017)]{annals:stats}
X.~Pennec.
\newblock Barycentric subspace analysis on manifolds.
\newblock \emph{Annals of Statistics}, to appear 2017.

\bibitem[Rossi and Ahmed(2015)]{nr}
R.~A. Rossi and N.~K. Ahmed.
\newblock The network data repository with interactive graph analytics and
  visualization.
\newblock In \emph{Proceedings of the Twenty-Ninth AAAI Conference on
  Artificial Intelligence}, 2015.
\newblock URL \url{http://networkrepository.com}.

\bibitem[Sanderson et~al.(1994)Sanderson, Donoghue, Piel, and
  Eriksson]{treebase}
M.~J. Sanderson, M.~J. Donoghue, W.~H. Piel, and T.~Eriksson.
\newblock {TreeBASE: a prototype database of phylogenetic analyses and an
  interactive tool for browsing the phylogeny of life}.
\newblock \emph{American Journal of Botany}, 81\penalty0 (6), 1994.

\bibitem[Sarkar(2011)]{sarkar}
R.~Sarkar.
\newblock Low distortion {Delaunay} embedding of trees in hyperbolic plane.
\newblock In \emph{Proc. of the International Symposium on Graph Drawing (GD
  2011)}, pages 355--366, Eindhoven, Netherlands, 2011.

\bibitem[Sibson(1978)]{Sibson1}
R.~Sibson.
\newblock Studies in the robustness of multidimensional scaling: Procrustes
  statistics.
\newblock \emph{Journal of the Royal Statistical Society, Series B},
  40\penalty0 (2):\penalty0 234--238, 1978.

\bibitem[Sibson(1979)]{Sibson2}
R.~Sibson.
\newblock Studies in the robustness of multidimensional scaling: Perturbational
  analysis of classical scaling.
\newblock \emph{Journal of the Royal Statistical Society, Series B},
  41\penalty0 (2):\penalty0 217--229, 1979.

\bibitem[Tay et~al.(2018)Tay, Tuan, and Hui]{tay2018hyperbolic}
Y.~Tay, L.~A. Tuan, and S.~C. Hui.
\newblock Hyperbolic representation learning for fast and efficient neural
  question answering.
\newblock In \emph{Proceedings of the Eleventh ACM International Conference on
  Web Search and Data Mining}, pages 583--591. ACM, 2018.

\bibitem[Verbeek and Suri(2016)]{verbeek2016metric}
K.~Verbeek and S.~Suri.
\newblock Metric embedding, hyperbolic space, and social networks.
\newblock \emph{Computational Geometry}, 59:\penalty0 1--12, 2016.

\bibitem[Walter(2004)]{walter2004}
J.~A. Walter.
\newblock {H-MDS}: a new approach for interactive visualization with
  multidimensional scaling in the hyperbolic space.
\newblock \emph{Information Systems}, 29\penalty0 (4):\penalty0 273--292, 2004.

\bibitem[Wilson et~al.(2014)Wilson, Hancock, Pekalska, and
  Duin]{wilson2014spherical}
R.C. Wilson, E.R. Hancock, E.~Pekalska, and R.~Duin.
\newblock Spherical and hyperbolic embeddings of data.
\newblock \emph{IEEE Transactions on Pattern Analysis and Machine
  Intelligence}, 36\penalty0 (11):\penalty0 2255--2269, 2014.

\bibitem[Zhang and Fletcher(2013)]{ProbPGA}
M.~Zhang and P.~Fletcher.
\newblock Probabilistic principal geodesic analysis.
\newblock In \emph{Advances in Neural Information Processing Systems 26 (NIPS
  2013)}, Lake Tahoe, NV, 2013.

\end{thebibliography}

\pagebreak

\appendix

\section{Glossary of Symbols}

\begin{table*}[h]
\centering
\begin{tabular}{l l}
\toprule
Symbol & Used for \\
\midrule
$x$, $y$, $z$ & vectors in the Poincar{\'e} ball model of hyperbolic space \\
$d_H$ & metric distance between two points in hyperbolic space \\
$d_E$ & metric distance between two points in Euclidean space \\
$d_U$ & metric distance between two points in metric space $U$ \\
$d$ & a particular distance value \\
$d_{i,j}$ & the distance between the $i$th and $j$th points in an embedding \\
$\mathbb{H}_r$ & the Poincar{\'e} ball model of $r$-dimensional Hyperbolic space \\
$r$ & the dimension of a Hyperbolic space \\
$\mathbb{H}$ & Hyperbolic space of an unspecified or arbitrary dimension \\
$\mathbb{M}_r$ & the Minkowski (hyperboloid) model of $r$-dimensional Hyperbolic space \\
$f$ & an embedding \\
$\mathcal{N}_a$ & neighborhood around node $a$ in a graph \\
$R_{a,b}$ & the smallest set of closest points to node $a$ in an embedding $f$ that contains node $b$ \\
$\text{MAP}(f)$ & the mean average precision fidelity measure of the embedding $f$ \\
$D(f)$ & the distortion fidelity measure of the embedding $f$ \\
$D_{\mathrm{wc}}(f)$ & the worst-case distortion fidelity measure of the embedding $f$ \\
$G$ & a graph, typically with node set $V$ and edge set $E$ \\
$T$ & a tree \\
$a, b, c$ & nodes in a graph or tree \\
$\operatorname{deg}(a)$ & the degree of node $a$ \\
$\operatorname{deg}_{\max}$ & maximum degree of a node in a graph \\
$\ell$ & the longest path length in a graph \\
$\tau$ & the scaling factor of an embedding \\
$\operatorname{reflect}_{x \rightarrow y}$ & a reflection of $x$ onto $y$ in hyperbolic space \\
$\operatorname{arg}(z)$ & the angle that the point $z$ in the plane makes with the $x$-axis \\
$X$ & matrix of points in hyperbolic space \\
$Y$ & matrix of transformed distances \\
% $S$ & diagonal scaling matrix used in h-MDS \\
% $v$ & vector of squared norm values used in h-MDS \\
% $u, \hat u$ & eigenvectors used in h-MDS \\
% $Z$ & reduced matrix used in h-MDS \\
% $\alpha, \beta$ & intermediate scalars used in h-MDS \\
$\gamma$ & geodesic used in PGA \\
$w_i$ & transformed points used in PGA \\
\bottomrule
\end{tabular}
\caption{Glossary of variables and symbols used in this paper.}
\label{table:glossary}
\end{table*}

\section{Related Work}
\label{sec:related}
Our study of representation tradeoffs for hyperbolic embeddings was motivated by
exciting recent approaches towards such embeddings in \citet{fb} and
\citet{ucl}.
Earlier efforts proposed using hyperbolic spaces for routing, starting with Kleinberg's work on geographic routing \cite{Kleinberg}.
\citet{Crovella} performed hyperbolic embeddings and routing for dynamic networks.
Recognizing that the use of hyperbolic space for routing required a large number of bits to store the vertex coordinates, \citet{Eppstein} introduced a scheme for succinct embedding and routing in the hyperbolic plane.
Another very recent effort also proposes using hyperbolic cones (similar to the cones that are the fundamental building block used in \citet{sarkar} and our work) as a heuristic for embedding entailment relations, i.e. directed acyclic graphs~\cite{ganea}.
The authors also propose to optimize on the hyperbolic manifold using its exponential map, as opposed to our approach of finding a closed form for the embedding should it exist (Section~\ref{sec:MDS}). An interesting avenue for future work is to compare both optimization methods empirically and theoretically, i.e., to understand the types of recovery guarantees under noise that such methods have.

There have been previous efforts to perform multidimensional scaling in hyperbolic space (the h-MDS problem), often in the context of visualization~\cite{lamping1994laying}. Most propose descent methods in hyperbolic space (e.g.~\cite{cvetkovski2016multidimensional}, \cite{walter2004}) and fundamentally differ from ours.
Arguably the most relevant is~\citet{wilson2014spherical}, which mentions exact recovery as an intermediate result, but ultimately suggests a heuristic optimization.
Our h-MDS analysis characterizes the recovered embedding and manifold and obtains the correctly centered one---a key issue in MDS.
For example, this allows us to properly find the components of maximal variation.
Furthermore, we discuss robustness to noise and produce optimization guarantees when a perfect embedding doesn't exist.

Several papers have studied the notion of hyperbolicity of networks, starting with the seminal work on hyperbolic graphs \citet{Gromov}. More recently, \citet{Mahoney} considered the hyperbolicity of small world graphs and tree-like random graphs. \citet{Dragan} performed a survey that examines how well real-world networks can be approximated by trees using a variety of tree measures and tree embedding algorithms. To motivate their study of tree metrics, \citet{Abraham} computed a measure of tree likeness on a Internet infrastructure network. 

We use matrix completion (closure) to perform embeddings with incomplete data. Matrix completion is a celebrated problem. \citet{TaoMatrix} derive bounds on the minimum number of entries needed for completion for a 
fixed rank matrix; they also introduce a convex program for matrix completion operating at near the optimal rate.

Principal geodesic analysis (PGA) generalizes principal components analysis (PCA) for the manifold setting. It was introduced and applied to shape analysis in \cite{PGA} and extended to a probabilistic setting in \cite{ProbPGA}.
There are other variants; the geodesic principal components analysis (GPCA) of~\citet{GPCA} uses our loss function.

\section{Low-Level Formulation Details}
We plan to release our PyTorch code, high precision solver, and other
routines on Github. A few comments are helpful to understand the
reformulation. In particular, we simply minimize the squared
hyperbolic distance with a learned scale parameter, $\tau$, e.g., :
\[ \min_{x_1,\dots,x_n,\tau}\sum_{1 \leq i < j \leq n} \left(\tau d_{H}(x_i,x_j) - d_{i,j}\right)^2 \]
We typically require that $\tau \geq 0.1$.

\begin{itemize}
\item On continuity of the derivative of the loss: Note that
  \[ \partial_{x} \mathsf{acosh}(1+x) = \frac{1}{\sqrt{(1+x)^2 -1 }}
= \frac{1}{\sqrt{x(x+2)}} \text{ hence } \lim_{x \to 0} \partial_{x}
\mathsf{acosh}(1+x) = \infty. \] Thus, $\lim_{y \to x} \partial_{x}
d_{H}(x,y) = \infty$. In particular, if two points happen to get near
to one another during execution, gradient-based optimization becomes
unstable. Note that $\exp\{\mathrm{acosh(1+x)}\}$ suffers from a
similar issue, and is used in both~\cite{fb, ucl}. This change may
increase numerical instability, and the public code for these
approaches does indeed take steps like masking out updates to mitigate
\textsc{NaN}s. In contrast, the following may be more stable:
\[ \partial_{x} \mathsf{acosh}(1+x)^2 = 2 \frac{\mathsf{acosh}(1+x)}{\sqrt{x(x+2)}} \text{ and in particular } \lim_{x \to 0} \partial_{x} \mathsf{acosh}(1+x)^2 = 2\]
The limits follows by simply applying L'Hopital's rule. In turn, this
implies the square formulation is continuously differentiable. Note
that it is not convex.
  \item One challenge is to make sure the gradient computed by PyTorch
    has the appropriate curvature correction (the Riemannian metric),
    as is well explained by \citet{fb}. The modification is
    straightforward: we create a subclass of \textsc{nn.Parameter}
    called \textsc{Hyperbolic\_Parameter}. This wrapper class allows
    us to walk the tree to apply the appropriate correction to the
    metric (which amounts to multiplying $\nabla_{w} f(w)$ by
    $\frac{1}{4}(1-\|w\|^2)^2$. After calling the \textsc{backward}
    function, we call a routine to walk the autodiff tree to find such
    parameters and correct them. This allows
    \textsc{Hyperbolic\_Parameter} and traditional parameters to be
    freely mixed.

    \item We project back on the hypercube following \citet{fb} and
      use gradient clipping with bounds of $[-10^{5},10^5]$. This
      allows larger batch sizes to more fully utilize the GPU.
\end{itemize}

\section{Combinatorial Construction Proofs}
\label{app:CombinatorialProofs}

\paragraph{Precision vs. model}

We first provide a simple justification of the fact (used in Section~\ref{sec:sarkar}) that representing distances $d$ requires about $d$ bits in hyperbolic space -- independent of the model of the space.
Formally, we show that the number of bits needed to represent a space depends only on the maximal and minimal desired distances and the geometry of the space.
Thus although the bulk of our results are presented in the Poincar{\'e} sphere, our discussion on precision tradeoffs is fundamental to hyperbolic space.

A representation using $b$ bits can distinguish $2^b$ distinct points in a space $S$.
Suppose we wish to capture distances up to $d$ with error tolerance $\varepsilon$ -- concretely, say every point in the ball $B(0,d)$ must be within distance $\varepsilon$ of a represented point.
By a sphere covering argument, this requires at least $\frac{V_S(d)}{V_S(\varepsilon)}$ points to be represented, where $V_S(r)$ is the volume of a ball of radius $r$ in the geometry.
Thus at least $b=\log \frac{V_S(d)}{V_S(\varepsilon)}$ bits are needed for the representation.
Notice that $V_E(d) \sim d^n$ in Euclidean $\R^n$ space, so this gives the correct bit complexity of $n \log(d/\varepsilon)$.
In hyperbolic space, $V_H$ is exponential instead of polynomial in $d$, so $O(d)$ bits are needed in the representation (for any constant tolerance).
In particular, this is independent of the model of the space.

\paragraph{Graph embedding lower bound}

Now, we derive a lower bound on the bits of precision required for embedding a graph into $\mathbb{H}_2$.
Afterwards we prove a result bounding the precision for our extension of Sarkar's construction for the $r$-dimensional Poincar\'{e} ball $\mathbb{H}_r$. Finally, we give some details on the algorithm for this extension.

We derive the lower bound by exhibiting an explicit graph and lower bounding the precision needed to represent its nodes (for any embedding of the graph into $\mathbb{H}_2$). The explicit graph $G_m$ we consider consists of a root node with $\text{deg}_{\max}$ chains attached to it. Each of these chains has $m$ nodes for a total of $1+m(\text{deg}_{\max})$ nodes, as shown in Figure~\ref{fig:chains}.

\begin{figure}
\centering
\begin{minipage}[b]{0.4\textwidth}
\begin{tikzpicture}[scale=1.0]
\node [circle,fill,inner sep=0pt,minimum size=5pt] (a0) at (0,1) [label={above:$a_0$}]{};
\node [circle,fill,inner sep=0pt,minimum size=5pt] (a1) at (-2,0) [label={below:$a_1$}]{};
\node [circle,fill,inner sep=0pt,minimum size=5pt] (a2) at (-1,0) [label={below:$a_2$}]{};
\node [circle,fill,inner sep=0pt,minimum size=5pt] (a3) at (0,0) [label={below:$a_3$}]{};    
\node [] (dots) at (1,0) [label={center:$\ldots$}]{};    
\node [circle,fill,inner sep=0pt,minimum size=5pt] (a4) at (2,0) [label={below:$a_{\text{deg}_{\max}}$}]{};

\draw (a0) -- (a1);
\draw (a0) -- (a2);
\draw (a0) -- (a3);
\draw (a0) -- (a4);

\end{tikzpicture}
  \end{minipage}
%  \hfill
  \begin{minipage}[b]{0.4\textwidth}
\begin{tikzpicture}[scale=1.0]
\node [circle,fill,inner sep=0pt,minimum size=5pt] (a0) at (0,1) [label={above:$a_0$}]{};
\node [circle,fill,inner sep=0pt,minimum size=5pt] (a1) at (-2,0) [label={left:$a_{1,1}$}]{};
\node [circle,fill,inner sep=0pt,minimum size=5pt] (a2) at (-1,0) [label={left:$a_{1,2}$}]{};
\node [circle,fill,inner sep=0pt,minimum size=5pt] (a3) at (0,0) [label={left:$a_{1,3}$}]{};    
\node [] (dots) at (1,0) [label={center:\ldots}]{};    
\node [circle,fill,inner sep=0pt,minimum size=5pt] (a4) at (2,0) [label={right:$a_{1,\text{deg}_{\max}}$}]{};

\node [circle,fill,inner sep=0pt,minimum size=5pt] (a5) at (-2,-1) [label={left:$a_{2,1}$}]{};
\node [circle,fill,inner sep=0pt,minimum size=5pt] (a6) at (-1,-1) [label={left:$a_{2,2}$}]{};
\node [circle,fill,inner sep=0pt,minimum size=5pt] (a7) at (0,-1) [label={left:$a_{2,3}$}]{};    
\node [] (dots2) at (1,-1) [label={center:\ldots}]{};    
\node [circle,fill,inner sep=0pt,minimum size=5pt] (a8) at (2,-1) [label={right:$a_{2,\text{deg}_{\max}}$}]{};

\node [circle,fill,inner sep=0pt,minimum size=20pt,color=white] (vdots1) at (-2,-2) [label={center:$\vdots$}]{};
\node [circle,fill,inner sep=0pt,minimum size=20pt,color=white] (vdots2) at (-1,-2) [label={center:$\vdots$}]{};
\node [circle,fill,inner sep=0pt,minimum size=20pt,color=white]  (vdots3) at (0,-2) [label={center:$\vdots$}]{};
\node [] (dots2) at (1,-2) [label={center:\ldots}]{};    
\node [circle,fill,inner sep=0pt,minimum size=20pt,color=white]  (vdots4) at (2,-2) [label={center:$\vdots$}]{};

\node [circle,fill,inner sep=0pt,minimum size=5pt] (a9) at (-2,-3) [label={below:$a_{m,1}$}]{};
\node [circle,fill,inner sep=0pt,minimum size=5pt] (a10) at (-1,-3) [label={below:$a_{m,2}$}]{};
\node [circle,fill,inner sep=0pt,minimum size=5pt] (a11) at (0,-3) [label={below:$a_{m,3}$}]{};    
\node [] (dots3) at (1,-3) [label={center:\ldots}]{};    
\node [circle,fill,inner sep=0pt,minimum size=5pt] (a12) at (2,-3) [label={below:$a_{m,\text{deg}_{\max}}$}]{};

\draw (a0) -- (a1);
\draw (a0) -- (a2);
\draw (a0) -- (a3);
\draw (a0) -- (a4);

\draw (a1) -- (a5);
\draw (a2) -- (a6);
\draw (a3) -- (a7);
\draw (a4) -- (a8);

\draw (a5) -- (vdots1);
\draw (a6) -- (vdots2);
\draw (a7) -- (vdots3);
\draw (a8) -- (vdots4);

\draw (vdots1) -- (a9);
\draw (vdots2) -- (a10);
\draw (vdots3) -- (a11);
\draw (vdots4) -- (a12);
    
\end{tikzpicture}
  \end{minipage}
\caption{Explicit graphs $G_m$ used to derive precision lower bound. Left: $m=1$ case (star graph). Right: $m>1$.}
\label{fig:chains}
\end{figure}
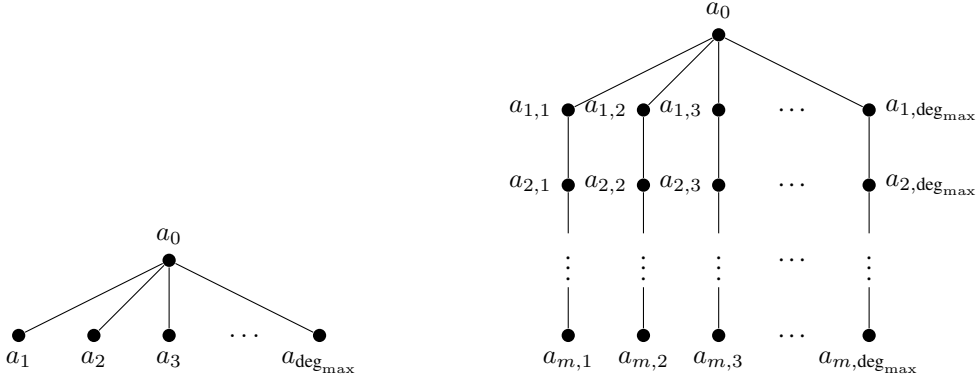

\begin{lemma}
The bits of precision needed to embed a graph with longest path $\ell$ is $\Omega\left(\frac{\ell}{\varepsilon} \log(\operatorname{deg}_{\max}) \right)$.
\end{lemma}
\begin{proof}

We first consider the case where $m=1$. Then $G_1$ is a star with $1+\text{deg}_{\max}$ children $a_1, a_2, \ldots, a_{\text{deg}_{\max}}$. Without loss of generality, we can place the root $a_0$ at the origin $0$. 

Let $x_i = f(a_i)$ be the embedding into $\mathbb{H}_2$ for vertex $a_i$ for $0 \leq i \leq \text{deg}_{\max}$.
We begin by showing that the distortion does not increase if we equalize the distances between the origin and each child $x_i$.
Let us write $\ell_{\max} = \max_i d_H(0,x_i) $ and $\ell_{\min} = \min_i  d_H(0,x_i)$. 

What is the worst-case distortion? We must consider the maximal expansion and the maximal contraction of graph distances.
Our graph distances are either 1 or 2, corresponding to edges ($a_0$ to $a_i$) and paths of length 2 ($a_i$ to $a_0$ to $a_j$).
By triangle inequality, $\frac{d_H(x_i,x_j)}{2} \leq \frac{d_H(0, x_i)}{2}+ \frac{d_H(0, x_j)}{2} \leq \ell_{\max}$.
This implies that the maximal expansion $\max_{i \neq j} d_H(f(a_i),f(a_j))/d_G(a_i,a_j)$ is $\frac{\ell_{\max}}{1} = \ell_{\max}$ occuring at a parent-child edge.
Similarly, the maximal contraction is at least $\frac{1}{\ell_{\min}}$. With this,
\[D_{\text{wc}}(f) \geq \frac{\ell_{\max}}{\ell_{\min}}.\]
Equalizing the origin-to-child distances (that is, taking $\ell_{\max} = \ell_{\min}$) reduces the distortion. Moreover, these distances are a function of the norms $\|x_i\|$, so we set $\|x_i\| = v$ for each child. 

Next, observe that since there are $\text{deg}_{\max}$ children to place, there exists a pair of children $x,y$ so that the angle formed by $x,0,y$ is no larger than $\theta = \frac{2\pi}{\text{deg}_{\max}}$. In order to get a worst-case distortion of $1+\varepsilon$, we need the product of the maximum expansion and maximum contraction to be no more than $1+\varepsilon$. The maximum expansion is simply $d_H(0,x)$ while the maximum contraction is $\frac{2}{d_H(x,y)}$, so we wan
\[2d_H(0,x) \leq (1+\varepsilon) d_H(x,y).\]

We use the log-based expressions for hyperbolic distance:
\[d_H(0,x) = \log \left(\frac{1+v}{1-v} \right),\]

and
\begin{align*}
d_H(x,y) &= 2 \log \left( \frac{\|x-y\| + \sqrt{\|x\|^2\|y\|^2 -2\langle x,y \rangle + 1}}{\sqrt{(1-\|x\|^2)(1-\|y\|^2)}} \right) \\
&=2 \log \left(\frac{\sqrt{2v^2(1-\cos \theta)} + \sqrt{v^4-2v^2 \cos \theta + 1}}{1-v^2} \right).
\end{align*}

This leaves us with 
\[  \log \left(\frac{\sqrt{2v^2(1-\cos \theta)} + \sqrt{v^4-2v^2 \cos \theta + 1}}{1-v^2} \right) (1+\varepsilon) \geq \log \left(\frac{1+v}{1-v} \right).\]

Now, since $1 > v^2$, we have that $\sqrt{2(1-\cos \theta)} \geq\sqrt{2v^2(1-\cos \theta)}$. Some algebra shows that  $\sqrt{3(1-\cos \theta)} \geq\sqrt{v^4 - 2v^2\cos \theta + 1}$, so that we can upper bound the left-hand side to write 
\[  \log \left(\frac{(1+\sqrt{\frac{3}{2}}) \sqrt{2(1-\cos \theta)}}{1-v^2} \right) (1+\varepsilon) \geq \log \left(\frac{1+v}{1-v} \right) .\]

Next we use the small angle approximation $\cos(\theta) = 1-\theta^2/2$ to get $\sqrt{2(1-\cos \theta)} = \theta$. Now we have
\[  \log \left(\frac{(1+\sqrt{\frac{3}{2}})  \theta }{1-v^2} \right) (1+\varepsilon) \geq \log \left(\frac{1+v}{1-v} \right) .\]

Since $v < 1$, $\frac{1}{1-v} > \frac{1}{1-v^2}$ and $\frac{1+v}{1-v} \geq \frac{1}{1-v}$, so we can upper bound the left-hand side and lower bound the right-hand side:
\[  \log \left(\frac{(1+\sqrt{\frac{3}{2}})  \theta }{1-v} \right) (1+\varepsilon) \geq  \log \left(\frac{1}{1-v} \right).\]

Rearranging,
\[ -\log(1-v) \geq -\log\left( \left(1+\sqrt{\frac{3}{2}}\right) \theta\right) \frac{1+\varepsilon}{\varepsilon}.\]

Recall that $\theta = \frac{2\pi}{\text{deg}_{\max}}$. Then we have that
\[-\log(1-v) \geq \left(\frac{1+\varepsilon}{\varepsilon} \right) \left(\log (\text{deg}_{\max}) - \log((2+\sqrt{6})\pi) \right),\]
so that
\[-\log(1-v) = \Omega\left(\frac{1}{\varepsilon} \log (\text{deg}_{\max}) \right).\]

Since $v = \|x\| = \|y\|$, $-\log(1-v)$ is precisely the required number of bits of precision, so we have our lower bound for the $m=1$ case. 

Next we analyze the $m>1$ case. Consider the embedded vertices $x_1, x_2, \ldots, x_m$ corresponding to one chain and $y_1, y_2, \ldots, y_m$ corresponding to another. There exists a pair of chains such that the angle formed by $x_m, 0, y_1$ is at most $\theta = \frac{2\pi}{\text{deg}_{\max}}$. Let $u = \|x_m\|$ and $v = \|y_1\|$. From the $m=1$ case, we have a lower bound on $-\log(1-v)$; we will now lower bound $-\log(1-u)$. The worst-case distortion we consider uses the contraction given by the path $x_m \rightarrow x_{m-1} \rightarrow \cdots \rightarrow x_1 \rightarrow 0 \rightarrow y_1$; this path has length $m+1$. The expansion is just the edge between $0$ and $y_1$. Then, to satisfy the worst-case distortion $1+\varepsilon$, we need
\[(m+1)d_H(0,y_1) \leq (1+\varepsilon)d_H(x_m,y_1).\]

Using the hyperbolic distance formulas, we can rewrite this as 
\[ 2 \log \left( \frac{\|x_m-y_1\| + \sqrt{\|x_m\|^2\|y_1\|^2 -2\langle x_m,y_1 \rangle + 1}}{\sqrt{(1-\|x_m\|^2)(1-\|y_1\|^2)}} \right) (1+\varepsilon) \geq (m+1) \log \left(\frac{1+v}{1-v} \right),\]

or,
\[ 2 \log \left( \frac{ \sqrt{u^2+v^2-2uv\cos \theta} + \sqrt{u^2v^2-2uv\cos\theta + 1}}{\sqrt{(1-u^2)(1-v^2)}} \right) (1+\varepsilon) \geq (m+1) \log \left(\frac{1+v}{1-v} \right).\]

Next,
\begin{align*}
2 \log &\left( \frac{ \sqrt{u^2+v^2-2uv\cos \theta} + \sqrt{u^2v^2-2uv\cos\theta + 1}}{\sqrt{(1-u^2)(1-v^2)}} \right) \\
 &\leq 2 \log \left( \frac{(1+\sqrt{\frac{3}{2}})\theta}{\sqrt{(1-u^2)(1-v^2)}} \right) =  \log \left( \frac{(1+\sqrt{\frac{3}{2}})^2\theta^2}{(1-u^2)(1-v^2)} \right) \\
&\leq  \log \left( \frac{(1+\sqrt{\frac{3}{2}})^2\theta^2}{(1-u)(1-v)} \right).
\end{align*}
In the first step, we used the same arguments as earlier. Applying this result and using $\frac{1+v}{1-v} \geq \frac{1}{1-v}$, we have
\[   \log \left( \frac{(1+\sqrt{\frac{3}{2}})^2\theta^2}{(1-u)(1-v)} \right)(1+\varepsilon) \geq (m+1) \log \left(\frac{1}{1-v} \right),\]
or,
\[   \log \left( \frac{(1+\sqrt{\frac{3}{2}})^2\theta^2}{1-u} \right)(1+\varepsilon) \geq (m-\varepsilon) \log \left(\frac{1}{1-v} \right).\]

Next we can apply the bound on $-\log(1-v)$.
\begin{align*}
\log\left(\frac{1}{1-u}\right) &\geq -\log\left((1+\sqrt{\frac{3}{2}})^2\theta^2\right) + \left( \frac{m-\varepsilon}{1+\varepsilon} \right) \log \left(\frac{1}{1-v} \right) \\
&\geq  -\log\left((1+\sqrt{\frac{3}{2}})^2\theta^2\right) +  \left( \frac{m-\varepsilon}{1+\varepsilon} \right) \left(\frac{1+\varepsilon}{\varepsilon} \right) \left(\log (\text{deg}_{\max}) - \log((2+\sqrt{6})\pi) \right)) \\
&= \left( \frac{m-\varepsilon}{\varepsilon} \right)  \log(\text{deg}_{\max}) - \left( \frac{m-\varepsilon}{\varepsilon} \right) \log((2+\sqrt{6})\pi) - \frac{1}{2}  \left(\log (\text{deg}_{\max}) - \log((2+\sqrt{6})\pi) \right).
\end{align*}
Here, we applied the relationship between $\theta$ and $\text{deg}_{\max}$ we derived earlier. To conclude, note that the longest path in our graph is $\ell = 2m$. Then, we have that 
\[-\log(1-u) = \Omega \left(\frac{\ell}{\varepsilon} \log(\text{deg}_{\max})\right),\]
as desired.
\end{proof}

\paragraph{Combinatorial construction upper bounds}

Next, we prove our extension of Sarkar's construction for $\mathbb{H}_r$, restated below.
 
{\bf Proposition 3.1.} \textit{The generalized $\mathbb{H}_r$ combinatorial construction has distortion at most $1+\varepsilon$ and requires at most $O(\frac{1}{\varepsilon}\frac{\ell}{r} \log \operatorname{deg}_{\max})$ bits to represent a node component for $r \leq (\log \operatorname{deg}_{\max})+1$, and $O(\frac{1}{\varepsilon}\ell)$ bits for $r > (\log \operatorname{deg}_{\max})+1$.}

\begin{proof}
%Our proof follows the lower-dimensional version of \citet{sarkar}. Briefly, the idea is that a child $Q$ is embedded in a cone emanating from parent $X$. The cone forms an angle $\alpha$, and due to the hyperbolic geometry, this angle satisfies $d_H(X,Q) = -\log \tan (\alpha/2)$. The angle $\alpha$ is determined by the degree of $X$, since we need a cone for each neighbor of $X$. Thus the degree determines the edge lengths. Moreover, to ensure $1+\varepsilon$ distortion, we take the largest such edge length $d$ and scale each edge by a common factor that ensures each edge is at least $((1+\varepsilon))/\varepsilon d$. 

The combinatorial construction achieves worst-case distortion bounded by $1+\varepsilon$ in two steps \cite{sarkar}.
First, it is necessary to scale the embedded edges by a factor of $\tau$ sufficiently large to enable each child of a parent node to be placed in a disjoint cone.
Note that there will be a cone with angle $\alpha$ less than $\frac{\pi}{\operatorname{deg}_{\max}}$.
The connection between this angle and the scaling factor $\tau$ is governed by $\tau = -\log( \tan \alpha/2)$.
As expected, as $\operatorname{deg}_{\max}$ increases, $\alpha$ decreases, and the necessary scale $\tau$ increases.

This initial step provides a Delaunay embedding (and thus a MAP of 1.0), but perhaps not sufficient distortion.
The second step is to further scale the points by a factor of $\frac{1+\varepsilon}{\varepsilon}$; this ensures the distortion upper bound. 

Our generalization to the Poincar\'{e} ball of dimension $r$ will modify the first step by showing that we can pack more children around a parent while maintaining the same angle.
In other words, for a fixed number of children we can increase the angle between them, correspondingly decreasing the scale.
We use the following generalization of cones for $\mathbb{H}_r$, defined by the maximum angle $\alpha \in [0,\pi/2]$ between the axis and any point in the cone.
Let cone $C(X,Y, \alpha)$ be the cone at point $X$ with axis $\vec{XY}$ and cone angle $\alpha$: $C(X,Y, \alpha) = \left\{ Z \in \mathbb{H}_{r} : \langle Z - X, Y - X \rangle \geq  \|Z-X\|\|Y-X\| \cos {\alpha} \right\}.$
We seek the maximum angle $\alpha$ for which $\operatorname{deg}_{\max}$ disjoint cones can be fit around a sphere.

Supposing $r-1 \le \log \operatorname{deg}_{\max}$, we use the following lower bound \cite{Jenssen} on the number of unit vectors $A(r,\theta)$ that can be placed on the unit sphere of dimension $r$ with pairwise angle at least $\theta$:
\[A(r, \theta) \geq (1+o(1)) \sqrt{2\pi r} \frac{\cos \theta}{(\sin \theta)^{r-1}}.\]

Consider taking angle
\[\theta = \asin({\operatorname{deg}_{\max}}^{-\frac{1}{r-1}}).\]
Note that
\[
  {\operatorname{deg}_{\max}}^{-\frac{1}{r-1}} = \exp \log {\operatorname{deg}_{\max}}^{-\frac{1}{r-1}} = \exp\left( -\frac{\log d}{r-1} \right) \le 1/e,
\]
which implies that $\theta$ is bounded from above and $\cos \theta$ is bounded from below.
Therefore
\[
  \operatorname{deg}_{\max} = \frac{1}{(\sin \theta)^{r-1}} \le O(1)\frac{\cos \theta}{(\sin \theta)^{r-1}} \le A(r,\theta).
\]
So it is possible to place $\operatorname{deg}_{\max}$ children around the sphere with pairwise angle $\theta$, or equivalently place $\operatorname{deg}_{\max}$ disjoint cones with cone angle $\alpha = \theta/2$.
Note the key difference compared to the two-dimensional case where $\alpha = \frac{\pi}{\operatorname{deg}_{\max}}$; here we reduce the angle's dependence on the degree by an exponent of $\frac{1}{r-1}$.

It remains to compute the explicit scaling factor $\tau$ that this angle yields; recall that $\tau = -\log( \tan \alpha/2)$ suffices~\cite{sarkar}.
We then have
% \begin{align*}
% \tau &= -\log(\tan \theta/2)  = -\log \left(\frac{\sin \theta}{1+\cos \theta} \right) = -\log \left(\frac{\sin \theta}{1 + \sqrt{1-\sin^2 \theta}} \right)\\
% &= -\log \left( \frac{1}{{\operatorname{deg}_{\max}}^{\frac{1}{r-1}} + \sqrt{{\operatorname{deg}_{\max}}^{\frac{2}{r-1}} - 1}}\right) \\ 
% &\leq  \log \left( 2\sqrt{{\operatorname{deg}_{\max}}^{\frac{2}{r-1}} - 1} \right) = \log 2 + O\left(\frac{1}{r} \log \operatorname{deg}_{\max }\right).
% \end{align*} 
\begin{align*}
  \tau &= -\log(\tan(\theta/4)) = -\log \left(\frac{\sin (\theta/2)}{1+\cos (\theta/2)} \right) = \log \left(\frac{1+\cos (\theta/2)}{\sin (\theta/2)} \right)
  \\&\le \log \left( \frac{2}{\sin(\theta/2)} \right) = \log \left( \frac{4\cos(\theta/2)}{\sin \theta} \right)
  \\&\le \log \left( \frac{4}{{\operatorname{deg}_{\max}}^{-\frac{1}{r-1}}} \right) = O\left(\frac{1}{r} \log \operatorname{deg}_{\max }\right)
  .
\end{align*}

 This quantity tells us the scaling factor without considering distortion (the first step).
 To yield the $1+\varepsilon$ distortion, we just increase the scaling by a factor of $\frac{1+\varepsilon}{\varepsilon}$.
 The longest distance in the graph is the longest path $\ell$ multiplied by this quantity.
    
Putting it all together, for a tree with longest path $\ell$, maximum degree $\operatorname{deg}_{\max}$ and distortion at most $1+\varepsilon$, the components of the embedding require (using the fact that distances $\|d\|$ require $d$ bits),
\[ O\left(      \frac{1}{\varepsilon}\frac{\ell}{r} \log d_{\max}  \right)\]
bits per component. % for $r  \leq (\log \operatorname{deg}_{\max}) + 1$.
This big-$O$ is with respect to $\operatorname{deg}_{\max}$ and any $r \le \log \operatorname{deg}_{\max} + 1$.

When $r > \log \operatorname{deg}_{\max} + 1$, $O\left( \frac{1}{\varepsilon}{\ell} \right)$ is a trivial upper bound.
Note that this cannot be improved asymptotically: As $\operatorname{deg}_{\max}$ grows, the minimum pairwise angle approaches $\pi/2$,%
\footnote{Given points $x_1, \dots, x_n$ on the unit sphere, $0 \le \| \sum x_i \|_2^2 = n + \sum_{i\neq j} \langle x_i, x_j \rangle$ implies there is a pair such that $x_i \cdot x_j \ge -\frac{1}{n-1}$, i.e. an angle bounded by $\cos^{-1}(-1/(n-1))$.}
so that $\tau = \Omega(1)$ irrespective of the dimension $r$.
% However, once we have increased the angles past $r = \log d_{\max}$ dimensions, the points cannot be further separated, and additional dimensions do not help.
\end{proof}

Next, we provide more details on the coding-theoretic child placement construction for $r$-dimensional embeddings. Recall that children are placed at the vertices of a hypercube inscribed into the unit hypersphere, with components in $\frac{\pm 1}{\sqrt{r}}$. These points are indexed by sequences ${a} \in \{0,1\}^r$ so that

\[{ x}_{ a} = \left( \frac{(-1)^{a_1}}{\sqrt{r}}, \frac{(-1)^{a_2}}{\sqrt{r}} , \ldots, \frac{(-1)^{a_r}}{\sqrt{r}} \right).\]

The Euclidean distance between ${ x}_{ a}$ and ${ x}_{ b}$ is a function of the Hamming distance $d_{\text{Hamming}}({ a},{ b})$ between ${ a}$ and ${ b}$. The Euclidean distance is exactly $2\sqrt{\frac{d_{\text{Hamming}}({ a},{ b})}{r}}$. Therefore, we can control the distances between the children by selecting a set of binary sequences with a prescribed minimum Hamming distance---a binary error-correcting code---and placing the children at the resulting hypercube vertices.

We introduce a small amount of terminology from coding theory. A binary code $\mathcal{C}$ is a set of sequences ${\bf a} \in \{0,1\}^r$. A $[r,k,h]_2$ code $\mathcal{C}$ is a binary linear code with length $r$ (i.e., the sequences are of length $r$), size $2^k$ (there are $2^k$ sequences), and minimum Hamming distance $h$ (the minimum Hamming distance between two distinct members of the code is $h$). 

The Hadamard code $\mathcal{C}$ has parameters $[2^k, k, 2^{k-1}]$. If $r=2^k$ is the dimension of the space, the Hamming distance between two members of $\mathcal{C}$ is at least $2^{k-1} = r/2$. Then, the distance between two distinct vertices of the hypercube ${ x}_{ a}$ and ${ x}_{ b}$ is $2\sqrt{\frac{r/2}{r}} = 2\sqrt{1/2} = \sqrt{2}$. Moreover, we can place up to $2^k=r$ points at least at this distance.

To build intuition, consider placing children on the unit circle ($r=2$) compared to the $r=128$-dimensional unit sphere. For $r=2$, we can place up to 4 points with pairwise distance at least $\sqrt{2}$. However, for $r=128$, we can place up to 128 children while maintaining this distance.

We briefly describe a few more practical details. Note that the Hadamard code is parametrized by $k$. To place $c+1$ children, take $k = \lceil \log_2(c+1) \rceil$. However, the desired dimension $r'$ of the embedding might be larger than the resulting code length $r=2^k$. We can deal with this by repeating the codeword. If there are $r'$ dimensions and $r|r'$, then the distance between the resulting vertices is still at least $\sqrt{2}$. Also, recall that when placing children, the parent node has already been placed. Therefore, we perform the placement using the hypercube, and rotate the hypersphere so that one of the $c+1$ placed nodes is located at this parent.

\paragraph{Embedding the ancestor transitive closure}
Prior work embeds the transitive closure of the WordNet noun hypernym graph \cite{fb}. Here, edges are placed between each word and its hypernym ancestors; MAP is computed over edges of the form (word, hypernym), or, equivalently, edges $(a,b)$ where $b\in \mathcal{A}(a)$ is an ancestor of $a$.

In this section, we show how to achieve arbitrarily good MAP on these types of transitive closures of a tree by embedding a weighted version of the tree (which we can do using the combinatorial construction with arbitrarily low distortion for any number dimensions). The weights are simply selected to ensure that nodes are always nearer to their ancestors than to any other node.

Let $T = (V,E)$ be our original graph. We recursively produce a weighted version of the graph called $T'$ that satisfies the desired property. Let $s$ be the depth of node $a \in V$. We weight each of the edges $(a,c)$, where $c$ is a child of $a$ with weight $2^s$. Now we show the following property: 

\begin{proposition}
Let $b \in \mathcal{A}(a)$ be an ancestor of $a$ and $e \not\in \mathcal{A}(a)$ be some node not an ancestor of $a$. Then,
\[d_G(a,b) < d_G(a,e).\]
\end{proposition}
\begin{proof}
Let $a$ be at depth $s$. First, the farthest ancestor from $a$ is the root, at distance $2^{s-1}+2^{s-2}+ \ldots + 2+1 = 2^s-1$. Thus $d_G(a,b) \leq 2^s-1$. 

If $e$ is a descendant of $a$, then $d_G(a,e)$ is at least $2^s$ Next, if $e$ is neither a descendant nor an ancestor of $a$, let $f$ be their nearest common ancestor, and let the depths of $a,e,f$ be $s,s_2,s_3$, respectively, where $s_3 < \min\{s_1,s_2\}$. We have that 
\begin{align*}
d_G(a,e) &= (2^{s-1}+\ldots+2^{s_3}) + (2^{s_2-1} + \ldots+2^{s_3}) \\
&= 2^{s} - 2^{s_3} + 2^{s_2} - 2^{s_3} \\
&=2^{s} + 2^{s_2} - 2^{s_3+1} \\
&\geq 2^{s} \\
&> d_G(a,b).
\end{align*}
The fourth line follows from $s_2 > s_3$. This concludes the argument.
\end{proof}

Therefore, embedding the weighted tree $T'$ with the combinatorial construction enables us to keep all of a word's ancestors nearer to it than any other word. This enables us to embed a transitive closure hierarchy (like WordNet's) while still embedding a nearly tree-like graph.
\footnote{Note that further separation can be achieved by picking weights with a base larger than $2$.}
Furthermore, the desirable properties of the construction still carry through (perfect MAP on trees, linear-time, etc).

%%% Local Variables:
%%% mode: latex
%%% TeX-master: "hyperbolic_arxiv"
%%% End:

\section{Proof of h-MDS Results}
\label{sec:mds-proof}

We first prove the condition that $X^T u = 0$ is equivalent to pseudo-Euclidean
centering.

\begin{proof}[Proof of Lemma~\ref{lmm:pe-centered}]
In the hyperboloid model, the variance term $\Psi$ can be written as
\begin{align*}
  \Psi(z; x_1, x_2, \ldots, x_n)
  &=
  \sum_{i=1}^k \sinh^2(d_H(x_i, z)) \\
  &=
  \sum_{i=1}^k \left( \cosh^2(d_H(x_i, z)) - 1 \right) \\
  &=
  \sum_{i=1}^k \left( (x_i^T Q z)^2 - 1 \right) \\
  &=
  \sum_{i=1}^k \left( (x_{0,i} z_0 - \vec{x}_i^T \vec{z})^2 - 1 \right) \\
  &=
  \sum_{i=1}^k \left( \left(x_{0,i} \sqrt{ 1 + \| \vec{z} \|^2 } - \vec{x}_i^T \vec{z} \right)^2 - 1 \right).
\end{align*}
The derivative of this with respect to $\vec{z}$ is
\begin{align*}
  \nabla_{\vec{z}} \Psi(z; x_1, x_2, \ldots, x_n)
  &=
  2 \sum_{i=1}^k 
  \left(x_{0,i} \sqrt{ 1 + \| \vec{z} \|^2 } - \vec{x}_i^T \vec{z} \right)
  \left(x_{0,i} \frac{\vec{z}}{\sqrt{ 1 + \| \vec{z} \|^2 }} - \vec{x}_i \right).
\end{align*}
At $\vec{z} = 0$ (or equivalently $z = e_0$), this becomes
\begin{align*}
  \left. \nabla_{\vec{z}} \Psi(z; x_1, x_2, \ldots, x_n) \right|_{\vec{z} = 0}
  &=
  2 \sum_{i=1}^k 
  \left(x_{0,i} \sqrt{ 1 + 0 } - 0 \right)
  \left(x_{0,i} \frac{0}{\sqrt{ 1 + 0 }} - \vec{x}_i \right) \\
  &=
  -2 \sum_{i=1}^k x_{0,i} \vec{x}_i.
\end{align*}
If we define the matrix $X \in \R^{n \times k}$ such that $X^T e_i = \vec{x}_i$ and the vector $u \in \R^k$ such that $u_i = x_{0,i}$, then
\begin{align*}
  \left. \nabla_{\vec{z}} \Psi(z; x_1, x_2, \ldots, x_n) \right|_{\vec{z} = 0}
  &=
  -2 \sum_{i=1}^k X^T e_i e_i^T u \\
  &=
  -2 X^T u.
\end{align*}
\end{proof}

\paragraph*{Centering and Geodesic Submanifolds}
A well-known property of the hyperboloid model is that the geodesic submanifolds on $\mathbb{M}_r$ are exactly the linear subspaces of $\R^{r+1}$ intersected with the hyperboloid model (Corollary A.5.5. from~\cite{Benedetti}).
This is analogous to how the affine subspaces of $\R^r$ are the linear subspaces of $\R^{r+1}$ intersected with the homogeneous-coordinates model of $\R^r$.
Notice that this directly implies that any geodesic submanifold can be written as a geodesic submanifold centered on any of the points in that manifold.
To be explicit with the definitions:
\begin{definition}
A geodesic submanifold is a subset $S$ of a manifold such that for any two points $x, y \in S$, the geodesic from $x$ to $y$ is fully contained within $S$.
\end{definition}

\begin{definition}
A geodesic submanifold rooted at a point $x$, given some local subspace of its tangent bundle $T$, is the subset $S$ of the manifold that is the union of all the geodesics through $x$ that are tangent at $x$ in a direction contained in $T$.
\end{definition}

Now we prove that centering with the pseudo-Euclidean mean preserves geodesic submanifolds.

First, we need the following technical lemma showing that projection to a manifold decreases distances.
\begin{lemma}
  \label{lmm:manifold-projection}
  Consider a dimension-$r$ geodesic submanifold $S$ and point $\bar x$ outside of it.
  Let $z$ be the projection of $\bar x$ onto $S$.
  Then for any point $x \in S$, $d_H(x, \bar x) > d_H(x, z)$.
\end{lemma}
\begin{proof}
  As a consequence of the projection, the points $x, z, \bar x$ form a right angle.
  From the hyperbolic Pythagorean theorem, we know that
  \[
    \cosh(d_H(x, \bar x)) = \cosh(d_H(x, z)) \cosh(d_H(z, \bar x)).
  \]
  Since $\cosh$ is increasing and at least $1$ (with equality only at $\cosh(0) = 1$), this implies that
  \[
    d_H(x, \bar x) > d_H(x, z).
  \]
\end{proof}

\begin{lemma}
If some points $x_1, \ldots, x_k$ lie in a dimension-$r$ geodesic submanifold $S$, then both a Karcher mean and a pseudo-Euclidean mean lie in this submanifold.
Equivalently, if the points lie in a submanifold, then this submanifold can be written as centered at the Karcher mean or the pseudo-Euclidean mean. 
\label{lemmaSubmanifoldCentering}
\end{lemma}
\begin{proof}
Suppose by way of contradiction that there is a Karcher mean $\bar x$ that lies outside this submanifold $S$.
Then, consider the projection $z$ of $\bar x$ onto $S$.
From Lemma~\ref{lmm:manifold-projection}, projecting onto $S$ has strictly decreased the distance to all the points on $S$.

As a result, the Frechet variance
\[
  \sum_{i=1}^k d_H^2(x_i, \bar x)
\]
also decreases when $\bar x$ is projected onto $S$.
From this, it follows that there is a minimum value of the Frechet variance (a Karcher mean) that lies on $S$.
An identical argument works for the pseudo-Euclidean distance, since the pseudo-Euclidean distance uses a variance that is just the sum of monotonically increasing functions of the hyperbolic distance.
\end{proof}

\begin{lemma}
Given some pairwise distances $d_{i,j}$, if it is possible to embed the distances in a dimension-$r$ geodesic submanifold rooted and centered at a pseudo-Euclidean mean, then it is possible to embed the distances in a dimension-$r$ geodesic submanifold rooted and centered at a Karcher mean, and vice versa.
\end{lemma}
\begin{proof}
Suppose that it is possible to embed the distances as some points $x_1, \ldots, x_k$ in a dimension-$r$ geodesic submanifold $S$.
Then, by Lemma~\ref{lemmaSubmanifoldCentering}, $S$ contains both a Karcher mean $\bar x$ and a pseudo-Euclidean mean $\bar x_P$ of these points.
If we reflect all the points such that $\bar x$ is reflected to the origin, then the new reflected points will also be an embedding of the distances (since reflection is isometric) and they will also be centered at the origin.
Furthermore, we know that they will still lie in a dimension-$r$ submanifold (now containing the origin) since reflection also preserves the dimension of geodesic submanifolds.
So the reflected points that we have constructed are an embedding of $d_{i,j}$ into a dimension-$r$ geodesic submanifold rooted and centered at a Karcher mean.
The same argument will show that (by reflecting $\bar x_P$ to the origin instead of $\bar x$) we can construct an embedding of $d_{i,j}$ into a dimension-$r$ geodesic submanifold rooted and centered at the pseudo-Euclidean mean.
This proves the lemma.
\end{proof}

%%% Local Variables:
%%% mode: latex
%%% TeX-master: "hyperbolic_arxiv"
%%% End:

\section{Perturbation Analysis}

\subsection{Handling Perturbations}
Now that we have shown that h-MDS recovers an embedding exactly, we
consider the impact of perturbations on the data. Given the necessity
of high precision for some embeddings, we expect that in some regimes
the algorithm should be very sensitive. Our results identify the
scaling of those perturbations.

First, we consider how to measure the effect of a perturbation on the
resulting embedding.  We measure the gap between two configurations of
points, written as matrices in $\mathbb{R}^{n \times r}$, by the sum
of squared differences $D(X,Y) = \tr((X-Y)^T(X-Y))$. Of course, this
is not immediately useful, since $X$ and $Y$ can be rotated or
reflected without affecting the distance matrix used for MDS--as these
are isometries, while scalings and Euclidean translations are
not. Instead, we measure the gap by
\[D_E(X,Y) = \inf \{D(X,PY) : P^T P = I\}.\]
In other words, we look for the configuration of $Y$ with the smallest
gap relative to $X$. For Euclidean MDS, \citet{Sibson1} provides an
explicit formula for $D_E(X,Y)$ and uses this formulation to build a
perturbation analysis for the case where $Y$ is a configuration
recovered by performing MDS on the perturbed matrix $XX^T+\Delta(E)$,
with $\Delta(E)$ symmetric.

\paragraph{Problem setup} In our case, the perturbations affect the hyperbolic distances. Let $H \in 
\mathbb{R}^{n \times  n}$ be the distance matrix for a set of points in hyperbolic space. Let $\Delta(H) \in \mathbb{R}^{n \times n}$ be the perturbation, with $H_{i,i}= 0$ and $\Delta(H)$ symmetric (so that $\hat H = H + \Delta_{H}$ remains symmetric).
% To simplify our derivations, we assume that the perturbation of $H$
% does not alter the dominant (Perron-Frobenius) eigenvalue and
% eigenvector of $Y$ from (\ref{eq:ydefn}).
The goal of our analysis is
to estimate the gap $D_E(X,Y)$ between $X$ recovered from $H$ with
h-MDS and $\hat X$ recovered from the perturbed distances
$H+\Delta(H)$.

\begin{lemma}
\label{lemma:hmds-perturb}
Under the above conditions, if $\lambda_{\min}$ denotes the smallest nonzero eigenvalue of $X X^T$ then up to second order in $\Delta(H)$,
\[
  D_E(X,\hat X) \le \frac{2 n^2}{\lambda_{\min}} \sinh^2\left( \| H \|_{\infty} \right) \| \Delta(H) \|_{\infty}^2.
\]
\end{lemma}

The key takeaway is that this upperbound matches our intuition for the
scaling: if all points are close to one another, then $\|H\|_{\infty}$
is small and the space is approximately flat (since $\sinh^2(z)$ is
dominated by $2z^2$ close to the origin). On the other hand, points at great
distance are sensitive to perturbations in an absolute sense.
% Note that far away points may be represented by having very large norms in
% some choices of coordinates.
% This leads us to consider other notions of recovery, which we do next.

\begin{proof}[Proof of Lemma~\ref{lemma:hmds-perturb}]
Similarly to our development of h-MDS, we proceed by accessing the underlying Euclidean distance matrix, and then apply the perturbation analysis from \citet{Sibson2}. There are three steps: first, we get rid of the $\acosh$ in the distances to leave us with scaled Euclidean distances. Next, we remove the scaling factors, and apply Sibson's result.
Finally, we bound the gap when projecting to the Poincar{\'e} sphere.
 
\paragraph*{Hyperbolic to scaled Euclidean distortion} Let $Y$ denote the scaled-Euclidean distance matrix, as in (\ref{eq:hmds-Y}), so that $Y_{i,j} = \cosh(H_{i,j})$. Let $\hat Y_{i,j} = \cosh(H_{i,j} + \Delta(H)_{i,j})$.
We write $\Delta(Y) = \hat Y - Y$ for the scaled Euclidean version of the perturbation. We can use the hyperbolic-cosine difference formula on each term to write
\begin{align*}
  \Delta(Y)_{i,j}
  &=
  \cosh(\hat H_{i,j}) - \cosh(H_{i,j}) \\
  &=
  (\cosh(H_{i,j} + \Delta(H)_{i,j}) - \cosh(H_{i,j})) \\
  &=
  2\sinh\left( \frac{ 2 H_{i,j} + \Delta(H)_{i,j} }{2} \right) \sinh\left( \frac{ \Delta(H)_{i,j} }{2} \right).
\end{align*}
In terms of the infinity norm, as long as $\| H \|_{\infty} \ge \| \Delta(H) \|_{\infty}$ (it is fine to assume this because we are only deriving a bound up to second order, so we can suppose that $\Delta(H)$ is small), we can simplify this to
% \begin{align*}
%   \| \Delta(Y) \|_{\infty}
%   &\le
%   \sinh\left( \frac{ 2 \| H \|_{\infty} + \| \Delta(H) \|_{\infty} }{2} \right) \sinh\left( \frac{ \| \Delta(H) \|_{\infty} }{2} \right).
% \end{align*}
% As long as $\| H \|_{\infty} \ge \| \Delta(H) \|_{\infty}$, we can simplify this to
\begin{align*}
  \| \Delta(Y) \|_{\infty}
  &\le
  2\sinh\left( \| H \|_{\infty} \right) \sinh\left( \| \Delta(H) \|_{\infty} / 2 \right).
\end{align*}

% on each term to write
% \begin{align*}
% (SE&+\Delta(SE))_{i,j} \leq \\
% & \frac{1}{2} (\cosh(H_{i,j}) -1) + \frac{1}{2}\left(\cosh\left(\frac{\Delta(H)_{i,j}}{H_{i,j}} \right) \right),
% \end{align*}
% or, 
% \[ \Delta(SE))_{i,j} \leq \frac{1}{2}\left(\cosh\left(\frac{\Delta(H)_{i,j}}{H_{i,j}} \right) \right).\]
%  Now we have the $\Delta(SE)$ perturbation in terms of $\Delta(H)$. Let us express this more conveniently with the $\infty$ norm. Let the smallest value of $H$ be denoted $h_{\min}$. Then,
%  \[ \|\Delta(SE)\|_{\infty} \leq\frac{1}{2}\cosh( \| \Delta(H)_\infty\| h_{\min}^{-1} ). \]
% %Note due to scaling the unit length, we can make this arbitrarirly small.

{\bf Scaled Euclidean to Euclidean inner product}.
Recall that if $X$ is the embedding in the hyperboloid model, then $Y = uu^T - XX^T$ from equation~\eqref{eq:hmds-Y2}, and furthermore $X^T u = 0$ so that $X$ can be recovered through PCA.
Now we are in the Euclidean setting, and can thus measure the result of the perturbation on the recovered $X$.
The proof of Theorem 4.1 in \citet{Sibson2} transfers to this setting.
This result states that if $\hat X$ is the configuration recovered from the perturbed inner products, then, the lowest-order term of the expansion of the error $D_E(X,\hat X)$ in the perturbation $\Delta(Y)$ is
\[
  D_E(X,\hat X) = \frac{1}{2} \sum_{j,k} \frac{(v_j^T \Delta(Y) v_k)^2}{\lambda_j + \lambda_k}.
\]
Here, the $\lambda_i$ and $v_i$ are the eigenvalues and corresponding orthonormal eigenvectors of $XX^T$ and the sum is taken over pairs of $\lambda_{j}, \lambda_k$ that are not both 0.
Let $\lambda_{\min}$ be the smallest nonzero eigenvalue of $XX^T$. Then,
\begin{align*}
  D_E(X,\hat X) 
  &\le 
  \frac{1}{2 \lambda_{\min}} \sum_{j,k} (v_j^T \Delta(Y) v_k)^2
  \le
  \frac{1}{2 \lambda_{\min}} \| \Delta(Y) \|_F^2 \\
  &\le
  \frac{n^2}{2 \lambda_{\min}} \| \Delta(Y) \|_{\infty}^2.
\end{align*}
Combining this with the previous bounds, and restricting to second-order terms in $\| \Delta(H) \|_{\infty}^2$ proves Lemma~\ref{lemma:hmds-perturb} for the embedding $X$ in the hyperboloid model.
\end{proof}

\paragraph{Projecting to the Poincar{\'e} disk}

Algorithm~\ref{alg:new_hmds} initially finds an embedding in $\mathbb{M}_r$, but optionally converts it to the Poincar{\'e} disk.
To convert a point $x$ in the hyperboloid model to $z$ in the Poincar{\'e} disk, take $z = \frac{x}{1 + \sqrt{1 + \|x\|_2^2}}$.
Let $Z \in \R^{n \times r}$ be the projected embedding.
Now we show that the same perturbation bound holds after projection.

\begin{lemma}
  \label{lmm:project-perturb}
  For any $x$ and $y$,
  $ \left\| \frac{x}{1 + \sqrt{1 + \|x\|_2^2}} - \frac{x}{1 + \sqrt{1 + \|x\|_2^2}} \right\| \le \| x-y \| $
\end{lemma}
\begin{proof}
  Let $u_x = \sqrt{1 + \|x\|^2}$ and define $u_y$ analogously.
  Note that $u_x \ge 2$, $u_x \ge \|x\|$, and
  \[
    u_y - u_x = \frac{u_y^2 - u_x^2}{u_y + u_x} = (\|y\|-\|x\|)\frac{\|y\|+\|x\|}{u_y + u_x} \le \|y\|-\|x\|.
  \]
  Combining these facts leads to the bound
  \begin{align*}
    \left\| \frac{x}{1 + \sqrt{1 + \|x\|_2^2}} - \frac{y}{1 + \sqrt{1 + \|y\|_2^2}} \right\| 
    &= \left\| \frac{x-y + x u_y - y u_y + y u_y - y u_x}{(1+u_x)(1+u_y)} \right\|
    \\&= \left\| \frac{(x-y)(1+u_y) + y(u_y - u_x)}{(1+u_x)(1+u_y)} \right\|
    \\&= \left\| \frac{x-y}{1+u_x} + \frac{y}{1+u_y}\frac{u_y-u_x}{1+u_x} \right\|
    \\&\le \frac{\left\| x-y \right\|}{1+u_x} + \frac{\left\| u_y-u_x \right\|}{1+u_x}
    \\&\le \left\| x-y \right\|.
  \end{align*}
\end{proof}

Lemma~\ref{lmm:project-perturb} is equivalent to the statement that $D(z, \hat z) \le D(x, \hat x)$ where $z, \hat z$ are the projections of $x, \hat x$.
Since orthogonal matrices $P$ preserve $\ell_2$ norm, $P\hat z$ is the projection of $P \hat x$ so $D(z, P \hat z) \le D(x, P \hat x)$ for any $P$.
Finally, $D(Z, P\hat Z)$ is just a sum over all columns and therefore $D(Z, P\hat Z) \le D(X, P\hat X)$.
This implies that $D_E(Z, \hat Z) \le D_E(X, \hat X)$ as desired.

\paragraph{The hyperbolic gap}
The gap $D(X,\hat X)$ can be written as a sum $\sum d_E(x_i, \hat{x}_i)^2$ over the vectors (columns) of $X,\hat X$.
We can instead ask about the hyperbolic gap
\[
  D_H(X, \hat X) = \inf \left\{ \sum d_H(x_i, P\hat{x}_i)^2 : P^T P = I \right\},
\]
which is a better interpretation of the perturbation error when recovering hyperbolic distances.

Note that for any points $x,y$ in the Gans model, we have
\[
  d_H(x,y) = \acosh\left( \sqrt{1 + \|x\|^2}\sqrt{1 + \|y\|^2} - \langle x,y \rangle \right) \le \acosh\left( \frac{2 + \|x\|^2 + \|y\|^2}{2} - \langle x,y \rangle \right) = \acosh\left( 1 + \frac{1}{2}\|x-y\|^2 \right).
\]
Furthermore, the function $\acosh(1 + t^2/2) - t$ is always negative except in a tiny region around $t=0$ (and attains a maximum here on the order of $10^{-10}$),
so effectively $\acosh\left( 1 + \frac{1}{2}\|x-y\|^2 \right) \le \|x-y\| = d_E(x,y)$,
and the same bound in Lemma~\ref{lemma:hmds-perturb} carries over to the hyperbolic gap.

%%% Local Variables:
%%% mode: latex
%%% TeX-master: "hyperbolic_arxiv"
%%% End:

\section{Proof of Lemma~\ref{lemma:pga}}

In this section, we prove Lemma~\ref{lemma:pga}, which gives a setting under which we can guarantee that the hyperbolic PGA objective is locally convex.

\begin{proof}[Proof of Lemma~\ref{lemma:pga}]
We begin by considering the component function
\[
  f_i(\gamma) = \acosh^2(1 + d_E^2(\gamma, v_i)).
\]
Here, the $\gamma$ is a geodesic through the origin.
We can identify this geodesic on the Poincar{\'e} disk with a unit vector $u$ such that $\gamma(t) = (2t-1) u$.
In this case, simple Euclidean projection gives us
\[
  d_E^2(\gamma, v_i) = \| (I - u u^T) v_i \|^2.
\]
Optimizing over $\gamma$ is equivalent to optimizing over $u$, and so
\[
  f_i(u) = \acosh^2\left(1 + \| (I - u u^T) v_i \|^2 \right).
\]
If we define the functions
\[
  h(\gamma) = \acosh^2(1 + \gamma)
\]
and
\[
  R(u) = \| (I - u u^T) v_i \|^2 =  \| v_i \|^2 - (u^T v_i)^2
\]
then we can rewrite $f_i$ as
\[
  f_i(u) = h(R(u)).
\]
Now, optimizing over $u$ is an geodesic optimization problem on the hypersphere.
Every goedesic on the hypersphere can be isometrically parameterized in terms of an angle $\theta$ as
\[
  u(\theta) = x \cos(\theta) + y \sin(\theta)
\]
for orthogonal unit vectors $x$ and $y$.
Without loss of generality, suppose that $y^T v_i = 0$ (we can always choose such a $y$ because there will always be some point on the geodesic that is orthogonal to $v_i$).
Then, we can write
\[
  R(\theta)
  =
  \| v_i \|^2 - (x^T v_i)^2 \cos^2(\theta)
  =
  \| v_i \|^2 - (x^T v_i)^2 + (x^T v_i)^2 \sin^2(\theta).
\]
Differentiating the objective with respect to $\theta$,
\begin{align*}
  \frac{d}{d \theta} h(R(\theta))
  &=
  h'(R(\theta)) R'(\theta) \\
  &=
  2 h'(R(\theta)) \cdot (v_i^T x)^2 \cdot \sin(\theta) \cos(\theta).
\end{align*}
Differentiating again,
\begin{align*}
  \frac{d^2}{d \theta^2} h(R(\theta))
  &=
  4 h''(R(\theta)) \cdot (v_i^T x)^4 \cdot \sin^2(\theta) \cos^2(\theta)
  +
  2 h'(R(\theta)) \cdot (v_i^T x)^2 \cdot \left( \cos^2(\theta) - \sin^2(\theta) \right).
\end{align*}
Now, suppose that we are interested in the Hessian at a point $z = x \cos(\theta) + y \sin(\theta)$ for some fixed angle $\theta$.
Here, $R(\theta) = R(z)$, and as always $v_i^T z = v_i^T x \cos(\theta)$, so
\begin{align*}
  \frac{d^2}{d \theta^2} h(R(\theta)) \big|_{u(\theta) = z}
  &=
  4 h''(R(\theta)) \cdot (v_i^T x)^4 \cdot \sin^2(\theta) \cos^2(\theta)
  +
  2 h'(R(\theta)) \cdot (v_i^T x)^2 \cdot \left( \cos^2(\theta) - \sin^2(\theta) \right) \\
  &=
  4 h''(R(z)) \cdot \frac{ (v_i^T z)^4 }{\cos^4(\theta)} \cdot \sin^2(\theta) \cos^2(\theta)
  +
  2 h'(R(z)) \cdot \frac{ (v_i^T x)^2 }{\cos^2(\theta)} \cdot \left( \cos^2(\theta) - \sin^2(\theta) \right) \\
  &=
  4 h''(R(z)) \cdot (v_i^T z)^4 \cdot \tan^2(\theta)
  +
  2 h'(R(z)) \cdot (v_i^T z)^2 \cdot \left( 1 - \tan^2(\theta) \right) \\
  &=
  2 h'(R(z)) \cdot (v_i^T z)^2 
  +
  \left(
    4 h''(R(z)) \cdot (v_i^T z)^4
    -
    2 h'(R(z)) \cdot (v_i^T z)^2
  \right)
  \tan^2(\theta).
\end{align*}
But we know that since $h$ is concave and increasing, this last expression in parenthesis must be negative.
It follows that a lower bound on this expression for fixed $z$ will be attained when $\tan^2(\theta)$ is maximized.
For any geodesic through $z$, the angle $\theta$ is the distance along the geodesic to the point that is (angularly) closest to $v_i$.
By the Triangle inequality, this will be no greater than the distance $\theta$ along the Geodesic that connects $z$ with the normalization of $v_i$.
On this worst-case geodesic,
\[
  v_i^T z = \|v_i\| \cos(\theta),
\]
and so
\[
  \cos^2(\theta) = \frac{(v_i^T z)^2}{\|v_i\|^2}
\]
and
\[
  \tan^2(\theta) = \sec^2(\theta) - 1 = \frac{\|v_i\|^2}{(v_i^T z)^2} - 1 = \frac{R(z)}{(v_i^T z)^2}.
\]
Thus, for any geodesic, for the worst-case angle $\theta$,
\begin{align*}
  \frac{d^2}{d \theta^2} h(R(\theta)) \big|_{u(\theta) = z}
  &\ge
  2 h'(R(z)) \cdot (v_i^T z)^2 
  +
  \left(
    4 h''(R(z)) \cdot (v_i^T z)^4
    -
    2 h'(R(z)) \cdot (v_i^T z)^2
  \right)
  \tan^2(\theta) \\
  &=
  2 h'(R(z)) \cdot (v_i^T z)^2 
  +
  \left(
    4 h''(R(z)) \cdot (v_i^T z)^2
    -
    2 h'(R(z))
  \right)
  R(z).
\end{align*}
From here, it is clear that this lower bound on the second derivative (and as a consequence local convexity) is a function solely of the norm of $v_i$ and the residual to $z$.
From simple evaluation, we can compute that
\[
  h'(\gamma) = 2 \frac{\acosh(1+\gamma)}{\sqrt{\gamma^2 + 2\gamma}}
\]
and
\[
  h''(x)
  =
  2 \frac{
    \sqrt{\gamma^2 + 2\gamma}
    -
    (1 + \gamma) \acosh(1 + \gamma)
  }{
    (\gamma^2 + 2\gamma)^{3/2}
  }.
\]
As a result
\begin{align*}
  4 \gamma h''(\gamma) + h'(\gamma)
  &=
  8 \frac{
    \gamma \sqrt{\gamma^2 + 2\gamma}
    -
    (\gamma^2 + \gamma) \acosh(1 + \gamma)
  }{
    (\gamma^2 + 2\gamma)^{3/2}
  }
  +
  2 \frac{
    (\gamma^2 + 2\gamma) \acosh(1+\gamma)
  }{
    (\gamma^2 + 2\gamma)^{3/2}
  } \\
  &=
  2 \frac{
    4 \gamma \sqrt{\gamma^2 + 2\gamma}
    -
    4 (\gamma^2 + \gamma) \acosh(1 + \gamma)
    +
    (\gamma^2 + 2\gamma) \acosh(1+\gamma)
  }{
    (\gamma^2 + 2\gamma)^{3/2}
  } \\
  &=
  2 \frac{
    4 \gamma \sqrt{\gamma^2 + 2\gamma}
    -
    (3 \gamma^2 + 2 \gamma) \acosh(1 + \gamma)
  }{
    (\gamma^2 + 2\gamma)^{3/2}
  }.
\end{align*}
For any $\gamma$ that satisfies $0 \le \gamma \le 1$,
\[
  4 \gamma \sqrt{\gamma^2 + 2\gamma}
  \ge
  (3 \gamma^2 + 2 \gamma) \acosh(1 + \gamma)
\]
and so
\[
  4 \gamma h''(\gamma) + h'(\gamma) \ge 0.
\]
Thus, if $0 \le R(z) \le 1$,
\begin{align*}
  \frac{d^2}{d \theta^2} h(R(\theta)) \big|_{u(\theta) = z}
  &\ge
  2 h'(R(z)) \cdot (v_i^T z)^2 
  +
  \left(
    4 h''(R(z)) \cdot (v_i^T z)^2
    -
    2 h'(R(z))
  \right)
  R(z) \\
  &=
  h'(R(z)) \cdot (v_i^T z)^2 
  +
  \left(
    4 h''(R(z)) \cdot R(z)
    +
    h'(R(z))
  \right) \cdot (v_i^T z)^2 
  -
  2 h'(R(z)) \cdot R(z) \\
  &\ge
  h'(R(z)) \cdot (v_i^T z)^2 
  -
  2 h'(R(z)) \cdot R(z) \\
  &=
  h'(R(z)) \cdot \left( \| v_i \|^2 - R(z) \right)
  -
  2 h'(R(z)) \cdot R(z) \\
  &=
  h'(R(z)) \cdot \left( \| v_i \|^2 - 3 R(z) \right).
\end{align*}
Thus, a sufficient condition for convexity is for (as we assumed above) $R(z) \le 1$ and
\[
  \| v_i \|^2 \ge 3 R(z).
\]
Combining these together shows that if
\[
  \acosh^2\left(1 + d_E(\gamma, v_i)^2 \right)
  =
  R(z)
  \le
  \min\left(1, \frac{1}{3} \| v_i \|^2 \right)
\]
then $f_i$ is locally convex at $z$.
The result of the lemma now follows from the fact that $f$ is the sum of many $f_i$ and the sum of convex functions is also convex.
\end{proof}

\section{Experimental Results}
In this section, we provide some additional experimental results. We also present results on an additional less tree-like graph (a search engine query response graph for the search term `California' \cite{ca-data}.)
 
\paragraph*{Combinatorial Construction: Parameters}
To improve the intuition behind the combinatorial construction, we report some additional parameters used by the construction. For each of the graphs, we report the maximum degree, the scaling factor $\nu$ that the construction used (note how these vary with the size of the graph and the maximal degree), the time it took to perform the embedding, in seconds, and the number of bits needed to store a component for $\varepsilon=0.1$ and $\varepsilon=1.0$.

\begin{table*}[h]
\centering
\begin{tabular}{|l|c|c||c|c|c|c|c|c|} \hline
                              &               &                  &                            &   &              \multicolumn{2}{c|}{$\varepsilon=0.1$} &   \multicolumn{2}{c|}{$\varepsilon=1.0$} \\ \hline
Dataset     	          &  Nodes & Edges   &   $d_{\max}$ & Time    & Scaling Factor &  Precision      & Scaling Factor &  Precision \\ \hline\hline
Bal. Tree 1          & 40  	      &  39          & 4             &  3.78      & 23.76                & 102   & 4.32 & 18 \\ \hline
Phylo. Tree          & 344      & 343        & 16           & 3.13       & 55.02               & 2361  &  10.00   &  412   \\ \hline \hline
WordNet              & 74374 &  75834  &  404       & 1346.62 & 126.11           & 2877  & 22.92 & 495\\ \hline
CS PhDs              & 1025     &  1043    &  46          &  4.99       & 78.30               & 2358  & 14.2 & 342 \\ \hline \hline
Diseases              & 516      & 1188     & 24          &    3.92       & 63.97             & 919 & 13.67 & 247  \\ \hline
Protein - Yeast   & 1458   & 1948     & 54           &  6.23       &  81.83             & 1413  & 15.02 & 273 \\ \hline \hline
Gr-QC                   & 4158    &  13428  & 68          & 75.41      & 86.90             &  1249 &  16.14 & 269 \\ \hline
California           & 5925     &   15770 &  105      & 114.41   & 96.46               & 1386  & 19.22 & 245\\ \hline
\end{tabular}
\caption{Combinatorial construction parameters and results.}
\label{table:comb_setup}
\end{table*}

\begin{table*}[]
\centering
\begin{tabular}{|c||c|c|c||c|c|c||c|c|c|}
\hline   
& \multicolumn{3}{c|}{MAP} &   \multicolumn{3}{c|}{2-MAP} &  \multicolumn{3}{c|}{$d_{avg}$}  \\ \hline
Rank       & h-MDS    & PCA      & FB      & h-MDS    & PCA      &  FB  & h-MDS    & PCA   & FB \\    \hline     \hline
Rank 2    & 0.346      &  0.614  &  {\bf 0.718}         & 0.754      &  {\bf 0.874}  &       0.802       & {\bf0.317}     &0.888  & 0.575  \\ \hline
Rank 5    & 0.439      & 0.627   &  {\bf 0.761}         & 0.844      &   0.905 & {\bf 0.950}       &{\bf 0.083}      &0.833 &  0.583 \\ \hline
Rank 10  & 0.471      & 0.632   &  {\bf 0.777}        & 0.857      &   0.912 &  {\bf 0.953}      & {\bf 0.048}      &0.804  & 0.586\\  \hline
Rank 50  & 0.560      & 0.687   &  {\bf 0.784}          & 0.880      &  0.962  &  {\bf 0.974}      & {\bf 0.036}     &0.768 & 0.584 \\ \hline
Rank 100 &0.645      & 0.698   &   {\bf 0.795}        & 0.926      &  {\bf 0.999}  &   0.981     &  {\bf 0.036}     &0.760 & 0.583 \\ \hline
Rank 200 &0.823      & {\bf 1.0}       &   0.811           & 0.968      & {\bf 1.0}      &   0.986      & {\bf 0.039}      & 0.746 & 0.583\\ \hline
\end{tabular}
\caption{Phylogenetic tree dataset.  Variation with rank, measured with MAP, 2-MAP, and $d_{avg}$. }
\label{table:rank_results}
\end{table*}

\paragraph*{Hyperparameter: Effect of Rank}
We also considered the influence of the dimension on the perfomance of
h-MDS, PCA, and FB. On the Phylogenetic tree dataset, we measured
distortion and MAP metrics for dimensions of 2,5,10,50,100, and
200. The results are shown in Table~\ref{table:rank_results}. We
expected all of the techniques to improve with better rank, and this
was the case as well. Here, the optimization-based approach typically
produces the best MAP, optimizing the fine details accurately. We
observe that the gap is closed when considering 2-MAP (that is, MAP
where the retrieved neighbors are at distance up to 2 away). In
particular we see that the main limitation of h-MDS is at the finest
layer, confirming the idea MAP is heavily influenced by local
changes. In terms of distortion, we found that h-MDS offers
good performance even at a very low dimension ($0.083$ at 5 dimensions).

\paragraph*{Precision Experiment}  (cf Table~\ref{table:mds-precision}).
Finally, we considered the effect of precision on h-MDS for a balanced tree and fixed dimension 10.
\begin{table}[ht!]
\centering
\begin{tabular}{|c||c|c|}
\hline 
  Precision   & $D_{avg}$ & MAP \\    \hline   
128 & 0.357  & 0.347 \\ \hline
256 & 0.091 &0.986 \\ \hline
512 & 0.076 & 1.0  \\  \hline
1024 &0.064 & 1.0  \\ \hline
\end{tabular}
\caption{h-MDS recovery at different precision levels for a $3$-ary tree and rank 10.}
\label{table:mds-precision}
\end{table}

\end{document}